\pdfoutput=1

\documentclass{article}

\usepackage[final, nonatbib]{neurips_2019}
\usepackage[square, sort, comma, numbers]{natbib}

\usepackage[utf8]{inputenc} 
\usepackage[T1]{fontenc}    
\usepackage{hyperref}       
\usepackage{url}            
\usepackage{booktabs}       
\usepackage{amsfonts}       
\usepackage{nicefrac}       
\usepackage{microtype}      

\usepackage{amsthm}
\usepackage{amsmath}
\usepackage{amssymb}
\usepackage{ upgreek }
\usepackage{dsfont}
\usepackage{mathrsfs}
\usepackage{neurips_2019}
\usepackage{graphicx}
\usepackage{mathtools}
\usepackage{xcolor}

\usepackage{tikz-cd}
\usetikzlibrary{cd}

\newtheorem{theorem}{Theorem}
\newtheorem{lemma}{Lemma}

\newtheorem{definition}{Definition}
\newtheorem{corollary}{Corollary}
\newtheorem{remark}{Remark}
\newtheorem{proposition}{Proposition}

\def\Gfun{\mathcal{G}}
\def\Qfun{\mathcal{Q}}

\definecolor{mydarkblue}{rgb}{0,0.08,0.55}
\definecolor{mypurple}{rgb}{0.5,0,0.15}
\usepackage{hyperref}
\hypersetup{
    colorlinks = true,
    linkcolor = blue,
    urlcolor  = black,
    citecolor = teal,
    anchorcolor = blue
}

\title{On the Equivalence between Graph Isomorphism Testing and Function Approximation with GNNs}

\author{%
  Zhengdao Chen \\
  Courant Institute of Mathematical Sciences\\
  New York University\\
  \texttt{zc1216@nyu.edu} \\
   \And
     Soledad Villar \\
  Courant Institute of Mathematical Sciences\\
  Center for Data Science \\
  New York University\\
  \texttt{soledad.villar@nyu.edu} \\ 
   \AND
     Lei Chen \\
  Courant Institute of Mathematical Sciences\\
  New York University\\
  \texttt{lc3909@nyu.edu} \\
   \And
    Joan Bruna \\
  Courant Institute of Mathematical Sciences\\
  Center for Data Science \\
  New York University\\
  \texttt{bruna@cims.nyu.edu} \\ 
}

\begin{document}

\maketitle

\begin{abstract}
Graph Neural Networks (GNNs) have achieved much success on graph-structured data. In light of this, there have been increasing interests in studying their expressive power. One line of work studies the capability of GNNs to approximate permutation-invariant functions on graphs, and another focuses on the their power as tests for graph isomorphism. Our work connects these two perspectives and proves their equivalence. We further develop a framework of the expressive power of GNNs that incorporates both of these viewpoints using the language of sigma-algebra, through which we compare the expressive power of different types of GNNs together with other graph isomorphism tests. In particular, we prove that the second-order Invariant Graph Network fails to distinguish non-isomorphic regular graphs with the same degree. Then, we extend it to a new architecture, Ring-GNN, which succeeds in distinguishing these graphs and achieves good performances on real-world datasets.
\end{abstract}

\section{Introduction}
Graph structured data naturally occur in many areas of knowledge, including computational biology, chemistry and social sciences. Graph Neural Networks, in all their forms, yield useful representations of graph data partly because they take into consideration the intrinsic symmetries of graphs, such as invariance and equivariance with respect to a reordering of the nodes \cite{scarselli2008graph, duvenaud2015convolutional, kipf2016semi, gilmer2017neural, hamilton2017inductive, velickovic2017graph, bronstein2017geometric, you2019pgnn}.
All these different architectures are proposed with different purposes (see \cite{wu2019comprehensive} for a survey and the references therein), and a priori it is not obvious how to compare their power. Recent works \cite{xu2018powerful, morris2019higher} proposed to study the representation power of GNNs via their performance on graph isomorphism tests. They proposed GNN models that are as powerful as the one-dimensional Weisfeiler-Lehman ($1$-WL) test for graph isomorphism \cite{weisfeiler1968reduction} and showed that no GNN based on neighborhood aggregation can be more powerful than the 1-WL test. 

Meanwhile, for feed-forward neural networks, many positive results have been obtained regarding their ability to approximate continuous functions, including the seminal results of the the universal approximation theorems \cite{cybenko1989approximation, hornik1991hornik}. Following this line of work, it is natural to study the expressive power of GNNs also in terms of function approximation, especially whether certain families of GNN can achieve universal approximation of continuous functions on graphs that are invariant to node permutations. 
Recent work \cite{maron2019universality} showed the universal approximation power of Invariant Graph Networks (IGNs), constructed based on the invariant and equivariant linear layers studied in \cite{maron2018invariant}, if the order of the tensor involved in the models are allowed to grow as the graph gets larger. 
However, these models are are not quite feasible in practice when the tensor order is high.

The first part of this work aims at building a bridge between graph isomorphism testing and invariant function approximation, the two main perspectives for studying the expressive power of GNNs. We demonstrate an equivalence between the the ability of a class of GNNs to distinguish between any pairs of non-isomorphic graph and its ability to approximate any (continuous) invariant functions on graphs. Furthermore, we show that it is natural to characterize the expressive power of function families on graphs via the sigma-algebras they generate on the graph space, allowing us to build a taxonomy of GNNs based on the inclusion relationships among the respective sigma-algebras.

Building on this theoretical framework, we identify an opportunity to increase the expressive power of the second-order Invariant Graph Network ($2$-IGN) in a tractable way by considering a ring of invariant matrices under addition and multiplication. We show that the resulting model, which we refer to as \emph{Ring-GNN}, is able to distinguish between non-isomorphic regular graphs where $2$-IGN provably fails. We illustrate these gains numerically in prediction tasks on synthetic and real graphs. 

Summary of main contributions:
\begin{itemize}
    \item We show the equivalence between testing for graph isomorphism and approximating of permutation-invariant functions on graphs as perspectives for characterizing the expressive power of GNNs.
    \item We further show that the expressive power of GNNs can be described by the sigma-algebra that they induce on the graph space, which unifies the two perspectives above and enables us to compare the expressive power of different GNNs variants.
    \item We propose Ring-GNN, a tractable extension of $2$-IGN that explores the ring of matrix addition and multiplication, which is more expressive than $2$-IGN and achieves good performances on practical tasks.
\end{itemize}

\section{Related Work}
\paragraph{Graph Neural Networks (GNNs) and graph isomorphism.} Graph isomorphism is a fundamental problem in theoretical computer science. 
It can be solved in quasi-polynomial-time \cite{babai2016graph}, but currently there is no known polynomial-time algorithm for solving it. For each positive integer $k$, the $k$-dimensional Weisfeiler-Lehman test ($k$-WL) is an iterative algorithm for determining if two graphs are isomorphic \cite{weisfeiler1968reduction}. $1$-WL is known to succeed on almost all pairs of non-isomorphic graphs \cite{babai1980random}, and the power of $k$-WL further increases as $k$ grows. The $1$-WL test has inspired the design of several GNN models \cite{hamilton2017inductive, zhang2017wlnm}.
Recently,  
\cite{xu2018powerful, morris2019higher} introduced graph isomorphism tests as a characterization of the power of GNNs and showed that if a GNN follows a neighborhood aggregation scheme, then it cannot distinguish pairs of non-isomorphic graphs that the $1$-WL test fails to distinguish. They also proposed particular GNN architectures that exactly achieves the power of the $1$-WL test by using multi-layer perceptrons (MLPs) to approximate injective set functions. 
A concurrent work \cite{maron2019provably} proves that the $k$th-order Invariant Graph Networks (IGNs) are at least as powerful as the $k$-WL tests, and similarly to our proposal, they augment the $2$-IGN model with matrix multiplication and show that the new model is at least as expressive as the $3$-WL test. \cite{murphy2019relational} proposed relational pooling (RP), an approach that combines \textit{permutation-sensitive} functions under all permutations to obtain a permutation-invariant function. If RP is combined with permutation-sensitive functions that are sufficiently expressive, then it can be shown to be a universal approximator. 
A drawback of RP is that its exact version is intractable computationally, and therefore it needs to be approximated by averaging over randomly sampled permutations, in which case the resulting functions is not guaranteed to be permutation-invariant. 

\paragraph{Universal approximation of functions with symmetry.} Many works have discussed the function approximation capabilities of neural networks that satisfy certain symmetries.
\cite{bloemreddy2019probabilistic} studied the probabilistic and functional symmetry in neural networks, and we discuss its relationship to our work in more detail in Appendix \ref{app.probabilistic}. \cite{ravanbakhsh2017sharing} showed that the equivariance of a neural network corresponds to symmetries in its parameter-sharing scheme. \cite{yarotsky2018universal} proposed a neural network architecture with polynomial layers that is able to achieve universal approximation of invariant or equivariant functions.  
\cite{maron2018invariant} studied the spaces of all invariant and equivariant linear functions, and obtained bases for such spaces. Building upon this work, \cite{maron2019universality} proposed the $G$-invariant networks for parameterizing functions invariant to a symmetry group $G$ on general domains. When we focus the graph domain with the natural symmetry group of node permutations, the model is also known as the $k$-IGNs, where $k$ represents the maximal order of the tensors involved in the model. It was shown in \cite{maron2019universality} that $k$-IGN achieves universal approximation of permutation-invariant functions on graphs if $k$ grows quadratically in the graph size, but such high-order tensors are prohibitive in practice. Upper bounds on the approximation power of the $G$-invariant networks when the tensor order is limited remains an open problem except for when the symmetry group is $A_n$ \cite{maron2019universality}. The concurrent work of \cite{keriven2019universal} provides an alternative proof and extends the result to the equivariant case, although it suffers from the same issue of possibly requiring high-order tensors. Specifically for learning in graphs, \cite{kondor2018covariant} proposed the compositional networks, which achieve equivariance and are inspired by the WL test. In the context of machine perception of visual scenes, \cite{herzig2018mapping} proposed an architecture that can potentially express all equivariant functions.

To the best our knowledge, this is the first work that shows an explicit connection between the two aforementioned perspectives of studying the representation power of GNNs - testing for graph isomorphism and approximating permutation-invariant functions on graphs. Our main theoretical contribution lies in showing an equivalence between them for under both finite and continuous feature spaces, with a natural generalization of the notion of graph isomorphism testing to the latter scenario. Then, we focus on the $2$-IGN model and prove that it cannot distinguish between non-isomorphic regular graphs with equal degrees. Hence, as a corollary, this model is not universal. Note that our result shows an \emph{upper} bound on the expressive power of $2$-IGN, whereas concurrently to us, \cite{maron2019provably} provides a \emph{lower} bound by relating $k$-IGNs to $k$-WL tests. In addition, similar to \cite{maron2019provably}, we also propose a modified version of $2$-IGN to capture higher-order interactions among the nodes without computing tensors of higher-order.

\section{Graph Isomorphism Testing and Universal Approximation}
\label{sec:giso_ua}
\paragraph{Notations} Let $G$ be a graph with $n$ nodes, $[n] \coloneqq \{1, ..., n \}$. The graph structure is characterized by the \emph{adjacency matrix}, $A \in \{0, 1\}^{n \times n}$, where $A_{i, j} = 1$ if and only if nodes $i$ and $j$ are connected by an edge. We allow additional information to be stored in the form of node and edge features, which we assume to belong to a compact set $\mathcal{X}\subset \mathbb R$, and we define $\Gfun \coloneqq \mathcal{X}^{n \times n}$. With an abuse of notation, we also regard $G$ as an $n \times n$ matrix that belongs to $\Gfun$, where $\forall i \in [n]$, $G_{i, i}$ is the feature of the $i$th node, and $\forall i, j \in [n]$, $G_{i, j}$ is the feature of the edge from node $i$ node $j$ (and equal to some predefined null value if $i$ and $j$ are not connected). For example, an undirected graph is represented by a symmetric matrix $G$. If there are no node or edge features (in other words, all nodes / edges are intrinsically indistinguishable), $G$ can be viewed as identical to $A$. Thus, modulo the permutation of the node orders, $\Gfun$ is the space of graphs with $n$ nodes. For two graphs $G, G' \in \Gfun$, we say they are \emph{isomorphic} (and write $G \simeq G'$) if $\exists \Pi \in S_n$ such that $\Pi^\intercal \cdot  G \cdot \Pi = G'$, where $S_n$ denotes the set of all $n \times n$ permutation matrices. A function $f$ on $\Gfun$ is called \textit{permutation-invariant} if $f(\Pi^\intercal \cdot G \cdot \Pi) = f(G)$, $\forall G \in \Gfun$, $\forall \Pi \in S_n$. 

A GNN with a graph-level scalar output represents a parameterized \emph{collection} of functions from $\Gfun$ to $\mathbb{R}$, which are typically permutation-invariant by design.
Given such a collection, we will show a close connection between its ability to approximate 
 permutation-invariant functions on $\Gfun$ and its power as graph isomorphism tests.  
First, we define these two properties precisely for any collection of permutation-invariant functions on $\Gfun$, denoted by $\mathcal{C}$:
\begin{definition}
\label{pd}
We say $\mathcal C$ is \textbf{GIso-discriminating} if $\forall G_1, G_2 \in \Gfun$ such that $G_1 \not\simeq G_2$,  $\exists h \in \mathcal{C}$ such that $h(G_1) \neq h(G_2)$. This definition is illustrated by Figure~\ref{fig:giso}. 
\end{definition}
\begin{figure}
\label{fig:giso}
    \centering
    \includegraphics[width=0.3\textwidth,trim={6cm 3cm 14cm 4cm},clip]{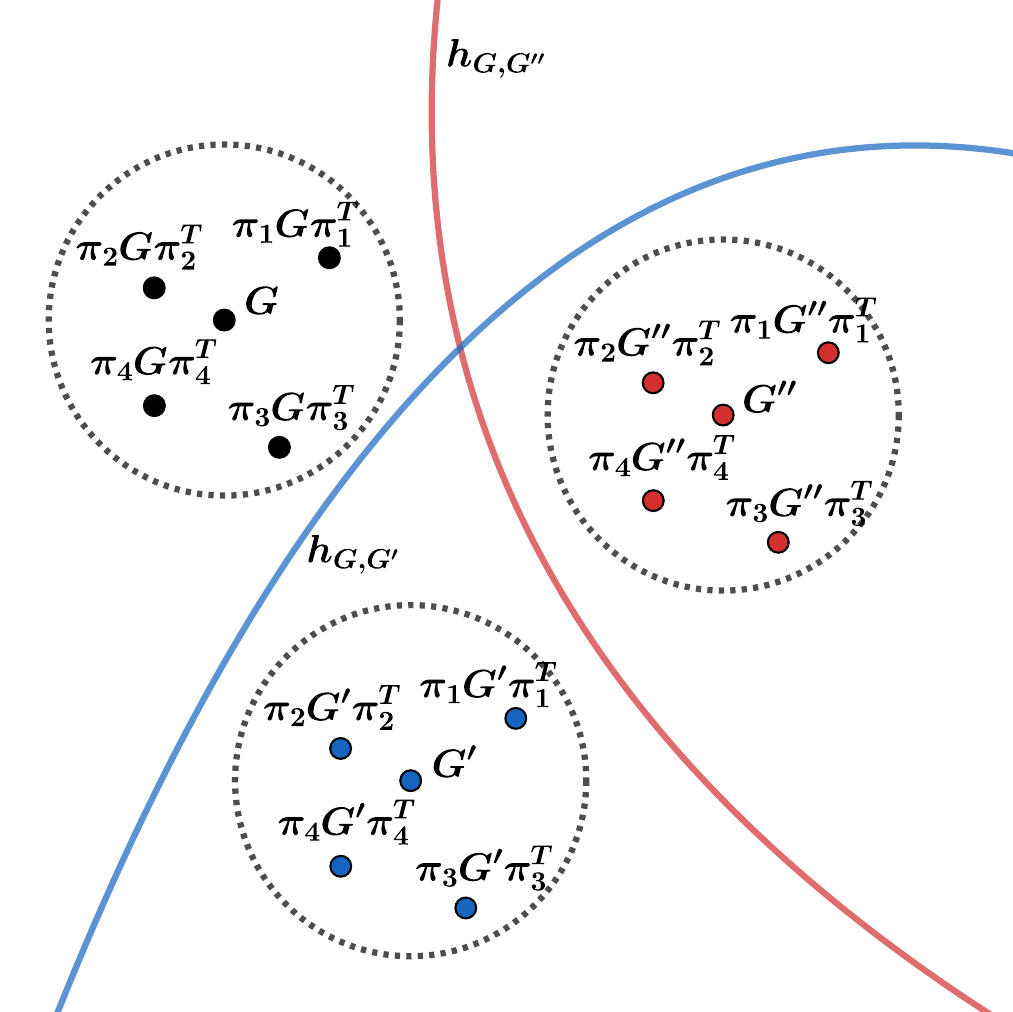}
\caption{Illustrating the definition of GIso-discriminating. $G, G'$ and $G''$ are mutually non-isomorphic, and each of the big circles with dashed boundary represents an equivalence class under graph isomorphism. $h_{G, G'}$ is a permutation-invariant function that obtains different values on equivalence class of $G$ and on that of $G'$, and similar $h_{G, G''}$. If the graph space has only these three equivalence classes of graphs, then $\mathcal{C} = \{h_{G, G'}, h_{G, G''} \}$ is GIso-discriminating.}
\end{figure}
\begin{definition}
\label{def:ua}
We say $\mathcal C$ is \textbf{universally approximating} if for all continuous permutation-invariant function $f: \Gfun \to \mathbb{R}$, and $\forall \epsilon > 0$, $\exists h \in \mathcal{C}$ such that $\| f - h \|_{\infty} := \sup_{G \in \Gfun} |f(G) - h(G)| < \epsilon$.
\end{definition}

\subsection{Finite feature space}
We first consider the simpler case where $\mathcal{X}$ is finite and show the equivalence between the two properties defined above.
\begin{theorem}
\label{UA2PD}
Suppose $\mathcal{X}$ is finite. If $\mathcal{C}$ is universally-approximating, then it is also GIso-discriminating.
\end{theorem}
\begin{proof}
As $\mathcal{X}$ is finite, we assume without loss of generality that it is a subset of $\mathbb{Z}$. Consider any pair of non-isomorphic graphs $G_1, G_2 \in \Gfun$. We define the \emph{$G_1$-indicator} function, $\mathds{1}_{\simeq G_1}: \Gfun \to \mathbb R$, by $\mathds{1}_{\simeq G_1}(G)=1$ if  $G \simeq G_1$ and 0 otherwise, which is evidently permutation-invariant. If $\mathcal{C}$ is universally-approximating, then $\mathds{1}_{\simeq G_1}$ can be approximated with precision $\epsilon = 0.1$ by some function $h \in \mathcal{C}$. Then, $h$ is a function that distinguishes $G_1$ from $G_2$.
\end{proof}

To obtain our result on the reverse direction, we introduce a concept of \emph{(NN-)augmented} collection of functions, which is especially natural when $\mathcal{C}$ itself is parameterized as neural networks.
It is defined to include any function that maps an input graph $G$ to $\mathcal{NN}([h_1(G), ..., h_d(G)])$, where $\mathcal{NN}$ is a feed-forward neural network (i.e. multi-layer perceptron) with $d$ as the input dimension and ReLU as the activation function, and $h_1, ..., h_d \in \mathcal{C}$. When $\mathcal{NN}$ is restricted to neural networks with at most $L$ layers, we denote the augmented collection by $\mathcal{C}^{+L}$. 

\begin{remark}
If $\mathcal{C}_{L_0}$ is the collection of feed-forward NNs with $L_0$ layers, then $\mathcal{C}_{L_0}^{+L}$ represents the collection of feed-forward NNs with $L_0 + L$ layers.
\end{remark}

\begin{remark}
If $\mathcal{C}$ is a collection of permutation-invariant functions, then so is $\mathcal{C}^{+L}$.
\end{remark}

Then, we are able to state the following result, which is proved in Appendix~\ref{app:pf_thm2}.\footnote{A later work \cite{chen2020can} contains a simpler proof.} 
\begin{theorem}
\label{PD2UAfin}
Suppose $\mathcal{X}$ is finite. If $\mathcal{C}$ is GIso-discriminating, then $\mathcal{C}^{+2}$ is universal approximating.
\end{theorem}

\subsection{Continuous feature space}
While graph isomorphism is inherently a discrete problem, 
the question of universal approximation is also interesting in cases where the input space is continuous.
We can naturally generalize the above results naturally to the scenarios of continuous input space under Definitions~\ref{pd} and \ref{def:ua}. 

\begin{theorem}\label{ua2pdinf}
Suppose $\mathcal{X}$ is a compact subset of $\mathbb{R}$. If $\mathcal{C}$ is universally approximating, then it is also GIso-discriminating.
\end{theorem}

\begin{theorem}
\label{thm:4}
Suppose $\mathcal{X}$ is a compact subset of $\mathbb{R}$. If $\mathcal{C}$ is a collection of continuous permutation-invariant functions on $\Gfun$ that is is GIso-discriminating, then $\mathcal{C}^{+2}$ is universally approximating.
\end{theorem}
These results are proved in Appendices~\ref{app:pf_thm3} and \ref{app:pf_thm4}.

\section{Characterizing Expressive Power through Sigma-Algebras}
\label{sec.sigma}
In this section, we assume for simplicity that $\mathcal{X}$ is a finite set, and hence so is $\Gfun$. Given a graph $G \in \Gfun$, we use $\mathcal{E}(G)$ to denote its \emph{isomorphism class}, defined as $\{ G' \in \Gfun: G' \simeq G \}$. Then, we let $\Qfun \coloneqq \Gfun /_\simeq $ denote the set of isomorphism classes in $\Gfun$, that is, 
$\Qfun \coloneqq \{ \mathcal{E}(G): G \in \Gfun \}$. The proofs of the theorems in this section are given in Appendix~\ref{sec.proofs.reformulating}.
\subsection{Defining sigma-algebras}

A maximally-expressive collection of permutation-invariant functions, $\mathcal{C}$, is one that allows us to identify exactly the isomorphism class $\tau$ that any given graph $G$ belongs to through the outputs of the functions $\mathcal{C}$ applied to $G$. 
In other words, each function on $\Gfun$ can be viewed as a``measurement'' that partitions the space $\Gfun$ into different subsets, where each subset contains all graphs in $\Gfun$ on which the function outputs the same value. Then, heuristically speaking, $\mathcal{C}$ being maximally expressive means that, collectively, the functions in $\mathcal{C}$ partition $\Gfun$ into exactly $\Qfun$.

This intuition can be formalized via the language of sigma-algebra. Recall that an \emph{algebra} on the set $\Gfun$ is a collection of subsets of $\Gfun$ that includes $\Gfun$ itself, is closed under complement, and is closed under finite union. In our case, as $\Gfun$ is assumed to be finite, it holds that an algebra on $\Gfun$ is also a \emph{sigma-algebra} on $\Gfun$, where a sigma-algebra further satisfies the condition of being closed under countable unions. If $\mathcal{S}$ is a collection of subsets of $\Gfun$, we let $\sigma(\mathcal{S})$ denote the sigma-algebra \emph{generated} by $\mathcal{S}$, defined as the smallest algebra that contains $\mathcal{S}$. Then, given a function $f$ on $\Gfun$, we also let $\sigma(f)$ denote the sigma-algebra generated by $f$, defined as the smallest sigma-algebra that contains all the pre-images under $f$ of the (Borel) sigma-algebra on $\mathbb{R}$. For more background on sigma-algebras and measurability, we refer the interested readers to standard texts on measure theory, such as \cite{bartle2014elements}.
Then, the following observation is straightforward:
\begin{proposition}
\label{prop:measurable}
A function $f$ on $\Gfun$ is permutation-invariant if and only if $f$ is measurable with respect to $\sigma(\Qfun)$.
\end{proposition}

\subsection{Graph isomorphism testing, universal approximation, and sigma-algebra inclusions}
\label{sec.reformulating}
Let $\mathcal{C}$ be a class of permutation-invariant functions on $\Gfun$. We further define the sigma-algebra generated by $\mathcal{C}$, denoted by $\sigma(\mathcal{C})$, as the smallest sigma-algebra that includes $\sigma(f)$ for all $f \in \mathcal{C}$. Then, Proposition~\ref{prop:measurable} implies that $\sigma(\mathcal{C}) \subseteq \sigma(\Qfun)$. Furthermore, as we shall demonstrate below, the expressiveness of $\mathcal{C}$ is reflected by the fineness of $\sigma(\mathcal{C})$, with a maximally-expressive $\mathcal{C}$ attaining $\sigma(\mathcal{C}) = \sigma(\Qfun)$:
\begin{theorem}\label{teo5}
Suppose $\mathcal{X}$ is finite. If $\mathcal{C}$ is GIso-discriminating, then $\sigma(\mathcal{C}) = \sigma(\Qfun)$.
\end{theorem}
Together with Theorem \ref{UA2PD}, the following is then an immediate consequence:
\begin{corollary}
Suppose $\mathcal{X}$ is finite. If $\mathcal{C}$ achieves universal approximation, then $\sigma(\mathcal{C}) = \sigma(\Qfun)$.
\end{corollary}
Conversely, we can also show that:
\begin{theorem} \label{teo6}
Suppose $\mathcal{X}$ is finite. If $\sigma(\mathcal{C}) = \sigma(\Qfun)$, then $\mathcal{C}$ is GIso-discriminating.
\end{theorem}

The framework of sigma-algebra not only allows us to reformulate the notions of being GIso-discriminating or universally-approximating, as shown above, but also allows us to compare the expressive power of different function families on graphs formally.
Given two classes of functions $\mathcal{C}_1$ and $\mathcal{C}_2$, we can formally define the statement ``$\mathcal{C}_1$ has less expressive power than $\mathcal{C}_2$'' as equivalent to $\sigma(\mathcal{C}_1) \subseteq \sigma(\mathcal{C}_2)$.
In Appendix \ref{app.comparison}, we use this notion to compare the expressive power of several different types of GNNs as well as other graph isomorphism tests like 1-WL and the linear programming relaxation, and the results are illustrated in Figure~\ref{fig.diagram}.

\begin{figure}[ht]
\label{diagram_main_text}
\small
\centering
\begin{tikzcd}
\text{sGNN}_1 \arrow[r,hook] & 1\text{-WL} \equiv \text{GIN} \equiv \text{LP} \arrow[d, hook] \arrow[dr, hook] &  \\
\text{adjacency spectrum} \arrow[dr,hook] & 2\text{-WL} \equiv 2\text{-IGN} \arrow[d,hook] & \text{SDP} \arrow[d,hook]\\
\text{sGNN}_J~ (J>1) \arrow[r, hook] & 3\text{-WL} \equiv \text{Ring-GNN / PPGN} \arrow[d, hook] &\text{SoS hierarchy} \\
& k\text{-WL}~ (k > 3) \equiv k\text{-GNN} \equiv k\text{-IGN} &
\end{tikzcd}
\caption{\small Comparison of function classes on graphs in terms of their expressive power under the sigma-algebra framework proposed in Section \ref{sec.sigma}, with details given in Appendix \ref{app.comparison}. For completeness, we have included relevant results that appeared later than the original publication of this work.}
\label{fig.diagram}
\end{figure}

\section{Ring-GNN: Exploring the Ring of Graph Operators with a GNN}
\subsection{Limitation of the $2$-Invariant Graph Network ($2$-IGN)}
We start by considering the \emph{Invariant Graph Networks} (IGNs, a.k.a. $G$-invariant networks) proposed in \cite{maron2019universality}, which are designed by interleaving permutation-equivariant linear layers (between tensors of potentially different orders) and point-wise nonlinear activation functions.
We present its definition in Appendix~\ref{app.Ginvariant} for completeness. In particular, for a positive integer $k$, an IGN model that involves tensors of order at most $k$ is called a $k$-IGN.
It is proved in \cite{maron2019universality} that $k$-IGN can achieve universality on graphs of size $n$ if $k$ grows quadratically in $n$, but less is known about its expressive power when $k$ is restricted. Meanwhile, in practice, it is difficult to implement an $k$-IGN when $k \geq 3$. Hence, our first goal is to examine its expressive power when $k = 2$.
The following result shows that $2$-IGN is not universal\footnote{In fact, it has later been proved that $k$-IGN is exactly as powerful as the $k$-WL test \cite{chen2020can, geerts2020expressive, geerts2022expressiveness}.}. The proof is given in Appendix \ref{app.Ginvariant}. 
\begin{theorem} \label{prop.Ginvariant}
$2$-IGN cannot distinguish among non-isomorphic regular graphs with the same degree.
\end{theorem}
For example, it cannot distinguish the pair of Circular Skip Link (CSL) graphs shown in Figure \ref{cslfig}.
\subsection{Ring-GNN as an extension of $2$-IGN}
\label{sec:ringgnn}
Given this limitation, we propose a novel GNN architecture that extends the $2$-IGN model without resorting to higher order tensors, which, specifically, are able to distinguish certain pairs of non-isomorphic regular graphs with the same degree. 

To gain intuition, we take 
the pair $G_{8, 2}$ and $G_{8,3}$ illustrated in Figure \ref{cslfig}
as an example.
Since all nodes in both graphs have the same degree, we are unable to break the symmetry among the nodes by simply updating the node hidden states through either neighborhood aggregation (as in $1$-WL and GIN) or through second-order permutation-equivariant linear layers (as in $2$-IGN). Meanwhile, however, nodes in the \textit{power graphs}\footnote{If $A$ is the adjacency matrix of a graph, its power graph has adjacency matrix $\min(A^2, 1)$. See Appendix~\ref{app:diagram}.} of $G_{8, 2}$ and $G_{8,3}$ have different degrees.
This observation motivates us to consider an augmentation of the $2$-IGN model that explores the polynomial \emph{ring} generated by the input matrix.
Then, together with point-wise nonlinear activation functions such as ReLU, power graph adjacency matrices like $\min(A^2, 1)$ can be expressed with suitable choices of model parameters. 

To define our model, we revisit the theory of linear equivariant functions developed in \cite{maron2018invariant}.
It is shown that any linear permutation-equivariant layer from $\mathbb{R}^{n \times n}$ to $\mathbb{R}^{n \times n}$ can be represented as $L_\theta(A) \coloneqq \sum_{p=1}^{15} \theta_p L_p(A) + \sum_{p=16}^{17} \theta_p \overline{L}_p$, where $\{L_p\}_{p \in [15]}$ is the set of 15 basis functions
for the space of linear permutation-equivariant functions from $\mathbb{R}^{n \times n}$ to $\mathbb{R}^{n \times n}$, $\overline{L}_{16}$ and $\overline{L}_{17} \in \mathbb{R}^{n \times n}$ are the basis for the permutation-equivariant bias terms, and $\theta \in \mathbb{R}^{17}$ are trainable parameters. Generalizing to maps from $\mathbb{R}^{n \times n \times d}$ to $\mathbb{R}^{n \times n \times d'}$, the full permutation-equivariant linear layer can be defined as $L_{\theta}(A)_{\cdot,\cdot,k'} \coloneqq \sum_{k=1}^d \big ( \sum_{p=1}^{15} \theta_{k, k', p} L_p(A_{\cdot, \cdot, k}) + \sum_{p=16}^{17} \theta_{k, k', p} \overline{L}_p \big )$, with $\theta \in \mathbb{R}^{d \times d' \times 17}$.

We now define a new architecture as follows. We let $A^{(0)} \in \mathbb R^{n\times n\times d_0}$ denote the generic input. For example, we can set $d_0=1$ and $A^{(0)} = G$. We fix some integer $T$ that denotes the number of layers, and for $t \in [T]$, we iteratively define
\begin{eqnarray*}
B_1^{(t)}&=& \sigma(L_{\alpha^{(t)}}(A^{(t-1)})) \\
B_2^{(t)}&=& \sigma(L_{\beta^{(t)}}(A^{(t-1)})\cdot L_{\gamma^{(t)}}(A^{(t-1)})) \\
A^{(t)}&=& k_1^{(t)} B_1^{(t)} + k_2^{(t)} B_2^{(t)}
\end{eqnarray*}
where $k_1^{(t)}, k_2^{(t)} \in \mathbb{R}$ and $\alpha^{(t)}, \beta^{(t)}, \gamma^{(t)} \in \mathbb{ R}^{d_{t-1} \times d_{t} \times 17}$ are trainable parameters, and $\sigma$ is the (entry-wise) ReLU function. 
If a graph-level scalar output is desired, then at the final layer, we compute the output as $\sum_{k=1}^{d_T} ( \theta_{k, 1} \sum_{i,j=1}^n \sigma(A_{i,j}^{(T)}) + \theta_{k, 2} \sum_{i=1}^n \sigma( A_{i,i}^{(T)} ) )$, where $\theta^{(T)} \in \mathbb{R}^{d_T \times 2}$ are additional trainable parameters. Note that this final layer is invariant to node permutations, and hence the overall model is also permutation-invariant. We call the resulting architecture the \textit{Ring-GNN}.

Note that each layer is equivariant, and the map from $A$ to the final scalar output is invariant. A Ring-GNN can reduce to a $2$-IGN if $k_2^{(t)} = 0$ for each $t$. With $J+1$ layers and suitable choices of the parameters, it is possible to expressed $\min(A^{2^{J}}, 1)$ in the $(J+1)^{th}$ layer. Therefore, we expect it to succeed in distinguishing certain pairs of regular graphs that the $2$-IGN fail on, such as the CSL graphs. Indeed, this is verified in the synthetic experiment presented in the next section. In addition, the normalized Laplacian matrix $\Delta \coloneqq I - D^{-1/2} A D^{-1/2}$ can also be approximated, since the degree matrix $D$ can be obtained from $A$ through a permutation-equivariant linear function, and then entry-wise inversion and square-root on the diagonal can be approximated by an MLP.

Computationally, the complexity of running the forward model grows as $O(n^{2.38})$ as $n$ increases, which is dominated by the matrix multiplication steps \cite{coppersmith1987arithmetic}.
In comparison, a $k$-IGN model will have complexity $\Omega(n^k)$. Therefore, the matrix multiplication steps enable the Ring-GNN to compute some higher-order interactions in the graph (which is neglected by $2$-IGN) while remaining relatively tractable computationally.
We also note that Ring-GNN can be augmented with matrix inverses or, more generally, functions on the spectrum of any of the intermediate representations. \footnote{When $A=A^{(0)}$ is the adjacency matrix of an undirected graph, one easily verifies that $A^{(t)}$ contains only symmetric matrices for all $t$.} 

\begin{figure}
    \label{cslfig}
    \centering
    \includegraphics[width=0.15\textwidth,trim={6cm 4.8cm 20cm 7.8cm},clip]{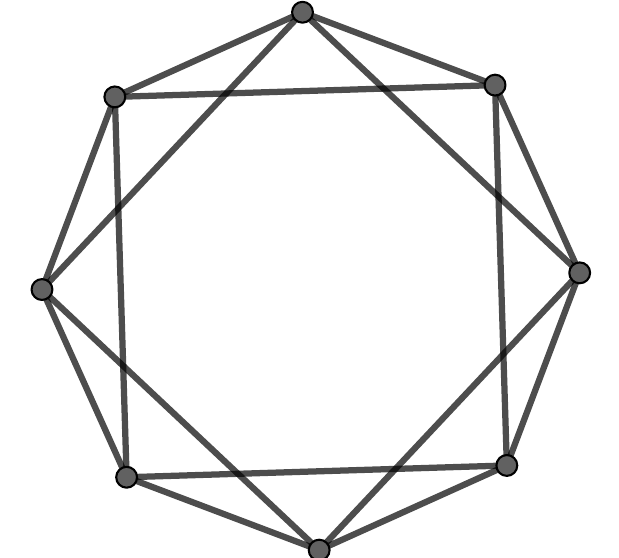}
    \includegraphics[width=0.15\textwidth,trim={6cm 4.8cm 20cm 7.8cm},clip]{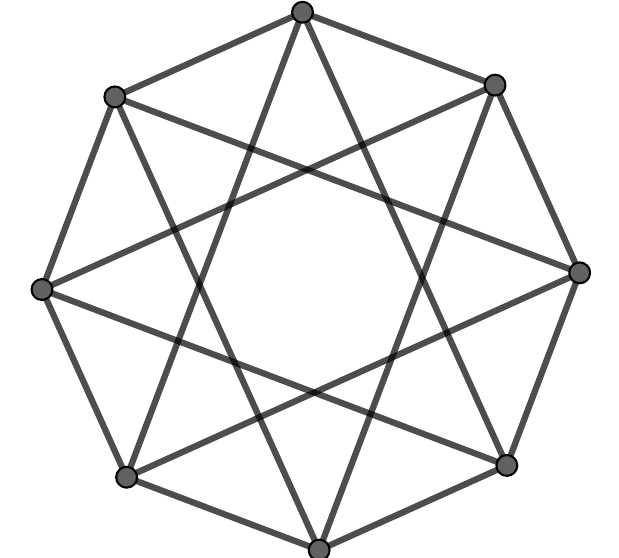}
    \caption{The Circular Skip Link (CSL) graphs $G_{n,k}$ are undirected graphs with $n$ nodes such that nodes $i$ and $j$ are connected if and only if $|i-j|\equiv 1 \text{ or } k \pmod n$. In this figure, we depict (left) $G_{8,2}$ and (right) $G_{8,3}$. It is easy to check that $G_{n,k}$ and $G_{n',k'}$ are not isomorphic unless $n=n'$ and $k\equiv \pm k' \pmod n$. Both $1$-WL and $2$-IGN fail to distinguish them.}
    \label{fig.skiplength}
\end{figure}

\section{Experiments}
\label{experiments}
The different models and the detailed setup of the experiments are discussed in Appendix \ref{archi}. All experiments are conducted on GeForce GTX 1080 Ti and RTX 2080 Ti.\footnote{The code is available at  \url{https://github.com/leichen2018/Ring-GNN}.}

\subsection{Classifying Circular Skip Links (CSL) graphs}
\label{cslexp}
The following experiment on synthetic data demonstrates the connection between function fitting and graph isomorphism testing. The CSL graphs\footnote{Link: \url{https://github.com/PurdueMINDS/RelationalPooling/tree/master/}.} are undirected regular graphs with node degree 4 \cite{murphy2019relational}, as illustrated in Figure \ref{cslfig}. Note that two CSL graphs $G_{n,k}$ and $G_{n',k'}$ are not isomorphic unless $n=n'$ and $k\equiv \pm k' \pmod n$. In the experiment, which has the same setup as in \cite{murphy2019relational}, we fix $n=41$ and set $k \in \{2, 3, 4, 5, 6, 9, 11, 12, 13, 16 \}$, and each $k$ corresponds to a distinct isomorphism class. The task is then to classify a graph $G_{n, k}$ by its skip length $k$.

Note that since the 10 classes have the same size, a naive uniform classifier should obtain $10\%$ accuracy. As we see from Table \ref{table.synthetic}, neither GIN and $2$-IGN outperform the naive classifier. Their failure in this task is unsurprising: $1$-WL is proved to fall short of distinguishing such pairs of non-isomorphic regular graphs \cite{cai1992optimal}, and hence neither can GIN \cite{xu2018powerful}; by Theorem~\ref{prop.Ginvariant}, $2$-IGN is unable to distinguish them either. Therefore, their empirical failure in this classification task is consistent with their theoretical limitations as graph isomorphism tests (and can also be understood as approximating the function that maps the graphs to their class labels).

It should be noted that, since graph isomorphism tests are not entirely well-posed as classfication tasks, the performance of GNN models can vary due to randomness. But the fact that Ring-GNNs achieve a relatively high \emph{maximum} accuracy indicates that, as a class of functions on graphs, it is rich enough to contain a good proportion of functions that distinguish the CSL graphs. 

\begin{table}[ht]
\centering
\begin{tabular}{l|lll||ll|ll}
\hline
& \multicolumn{3}{|c||}{CSL} & \multicolumn{2}{c|}{IMDB-B} & \multicolumn{2}{c}{IMDB-M} \\
GNN architecture              & max  & min & std  & mean & std & mean & std \\
\hline \hline
RP-GIN $\dagger$              & 53.3 & 10  & 12.9 & -     & -     & -     & -     \\
GIN $\dagger$ $\ddagger$       & 10   & 10  & 0    & 75.1  & 5.1   & 52.3  & 2.8   \\
$2$-IGN $\dagger$ & 10   & 10  & 0    & 71.3 & 4.5   & 48.6 & 3.9   \\
$\text{sGNN}_5$                        & 80   & 80  & 0    & 72.8  & 3.8   & 49.4  & 3.2   \\
$\text{sGNN}_2$                        & 30   & 30  & 0    & 73.1  & 5.2   & 49.0    & 2.1   \\
$\text{sGNN}_1$                        & 10   & 10  & 0    & 72.7  & 4.9   & 49.0    & 2.1   \\
LGNN \cite{chen2019cdsbm}                         & 30   & 30  & 0    & 74.1  & 4.6   & 50.9  & 3.0     \\
Ring-GNN                          & 80   & 10  & 15.7    & 73.0  & 5.4   & 48.2  & 2.7  \\
Ring-GNN (w/ degree) $\ddagger$  &   -    &   -    &   -    &  73.3   & 4.9  & 51.3  & 4.2 \\
\hline
\end{tabular}
\vspace{5pt}
\caption{\textbf{(left)} Accuracy of different GNNs at classifying CSL graphs(see Section \ref{cslexp}). We report the best and worst performances among 10 experiments.
\textbf{(right)} Accuracy of different GNNs in classication tasks on the two IMDB datasets (see Section \ref{sec.imbdb}). We report the best performance among all 350 epochs on 10-fold cross-validation, following in \cite{xu2018powerful}. 
$\dagger$: Reported in \cite{murphy2019relational}, \cite{xu2018powerful} and \cite{maron2018invariant}. 
$\ddagger$: On the IMDB datasets, both GIN and the Ring-GNN (w/ degree) take the node degrees as input node features (see Section \ref{sec.imbdb}). 
}
\label{table.synthetic}
\end{table}

\subsection{IMDB datasets} \label{sec.imbdb}
We use the two IMDB datasets (IMDB-B and IMDB-M)\footnote{Link: \url{https://github.com/weihua916/powerful-gnns/blob/master/dataset.zip}.} to test different models in real-world social networks. Since our focus is on distinguishing graph structures, these datasets are convenient as they do not contain node or edge features \cite{yanardag2015deep}.
IMDB-B has 1000 graphs, with 19.8 nodes per graph on average and 2 class labels. IMDB-M has 1500 graphs, with 13.0 nodes per graph on average and 3 class labels. 
Both datasets are randomly partitioned into training and validation sets with ratio $9 : 1$. 
As these two datasets have no informative node features, GIN uses one-hot encoding of the node degrees as input node features, while the other baseline models treat all nodes as having identical features. For a fairer comparison, we apply two versions of Ring-GNN: the first one treats all nodes as having identical input features, denoted as \emph{Ring-GNN}; the second one uses the node degree as input features (though not via one-hot encoding, due to computational constraints, but simply as one integer per node), denoted as \emph{Ring-GNN w/ degree}. All models are evaluated via 10-fold cross validation and the best accuracy is calculated through averaging across folds and then maximizing along epochs, following \cite{xu2018powerful}. Table \ref{table.synthetic} shows that the Ring-GNN models achieve higher or similar performance compared to $2$-IGN on both datasets, and slightly worse performance compared to GIN. 

\subsection{Other real-world datasets}
We perform further experiments on four other real-world datasets for classification tasks, including a social network dataset, COLLAB, and three bioinformatics datasets, MUTAG, PTC, PROTEINS\footnote{Same link as above.}~\cite{yanardag2015deep}. The experiment setup (10-fold cross validation, training/validation split) is identical to that of the IMDB datasets, except that all the bioinformatics datasets contain node features, and more details of hyperparameters are included in Appendix~\ref{archi}. As shown in Table \ref{tab.other}, Ring-GNN outperforms $2$-IGN in all four datasets, and outperforms GIN in one out of the four datasets. 

\begin{table}[ht]
\centering
\begin{tabular}{l|l|l|l|l}
\hline
            & COLLAB & MUTAG & PTC   & PROTEINS \\ \hline
Ring-GNN     & 80.1$\pm$1.4   & 86.8$\pm$6.4  & 65.7$\pm$7.1  & 75.7$\pm$2.9     \\ \hline
GIN $\dagger$          & 80.2$\pm$1.9   & 89.4$\pm$5.6  & 64.6$\pm$7.0  & 76.2$\pm$2.8     \\ \hline
$2$-IGN$\dagger$  & 77.9 $\pm$ 1.7  & 84.6$\pm$10.0 & 59.5$\pm$7.3 & 75.2$\pm$4.3   \\ \hline
\end{tabular}
\vspace{5pt}
\caption{Accuracy of different GNNs evaluated on several other real-world datasets.  We report the best performance among all epochs on 10-fold cross-validation. $\dagger$: Reported in \cite{xu2018powerful} and \cite{maron2018invariant}. }
\label{tab.other}
\end{table}

\section{Conclusions}
In this work, we address the important question of organizing the fast-growing zoo 
of GNN architectures in terms of what functions they can and cannot represent. 
We follow the approach via graph isomorphism tests and show that is equivalent 
to the classical perspective of function approximation. 
We leverage the theoretical insights to augment the $2$-IGN model
with the ring of operators associated with matrix multiplication, which gives provable gains in expressive power with tractable computational complexity and is amenable to further efficiency improvements by leveraging sparsity in the graphs. 

Our general framework leaves many interesting questions unresolved. First, our current GNN taxonomy is still incomplete. Second, we need a deeper analysis on which elements of the ring of operators created in Ring-GNN are really relevant for different types of applications.
Finally, beyond strict graph isomorphism, a natural next step is to consider weaker metrics in the space of graphs, such as the Gromov-Hausdorff distance,
which could better reflect the requirements on the stability of powerful graph representations to small graph perturbations in real-world applications \cite{gama2019stability}. 

\paragraph{Acknowledgements} We thank Haggai Maron and Thomas Kipf for fruitful discussions that pointed us towards Invariant Graph Networks as powerful models to study representational power in graphs. We thank Sean Disar\`{o} for pointing out errors in an earlier version of this paper.
We also thank Michael M. Bronstein for supporting this research with computing resources.
This work was partially supported by NSF grant RI-IIS 1816753, NSF CAREER CIF 1845360, the Alfred P. Sloan Fellowship, Samsung GRP and Samsung Electronics.
SV was partially funded by EOARD FA9550-18-1-7007 and the Simons Collaboration Algorithms and Geometry.

\newpage
\bibliographystyle{plain}
\bibliography{ref}

\newpage
\appendix
\section{Proofs for Section~\ref{sec:giso_ua}}
\label{app.universal}
\subsection{Proof of Theorem~\ref{PD2UAfin}}
\label{app:pf_thm2}
Theorem~\ref{PD2UAfin} is a consequence of the two following lemmas.
\begin{lemma}
\label{lemma1}
    If $\mathcal{C}$ is GIso-discriminating, then for all $G \in \Gfun$, there exists a function $\tilde{h}_G \in \mathcal{C}^{+1}$ such that for all $G', \tilde{h}_G(G') = 0$ if and only if $G \simeq G'$.
\end{lemma}

\begin{lemma}
\label{lemma2}
Let $\mathcal C$ be a class of permutation-invariant functions from $\Gfun$ to $\mathbb R$ so that for all $G \in \Gfun$, there exists $\tilde{h}_G \in \mathcal{C}$ satisfying $\tilde{h}_G(G') = 0$ if and only if $G \simeq G'$. 
Then  $\mathcal{C}^{+1}$ is universally approximating.
\end{lemma}

\subsubsection{Proof of Lemma \ref{lemma1}}
Given $G, G' \in \Gfun$ with $G \not\simeq G'$, we let $h_{G, G'} \in \mathcal{C}$ be a function that distinguishes this pair, i.e. $h_{G, G'} (G) \neq h_{G, G'} (G')$. Then, we define a function $\overline{h}_{G, G'}$ by $ \overline{h}_{G, G'}(\cdot) = |h_{G, G'}(\cdot) - h_{G, G'}(G)| $.
We see that $\forall G'' \in \Gfun$, if $G'' \simeq G$, then $h_{G, G'}(G'') = h_{G, G'}(G)$, and so $\overline{h}_{G, G'}(G'') = 0$; if $G'' \simeq G'$, then $\overline{h}_{G, G'}(G'') > 0$; otherwise, $\overline{h}_{G, G'}(G'') \geq 0$.

Next, we define a function $\tilde{h}_G$ by $\tilde{h}_G(\cdot) = \mathop{\sum_{G' \in \Gfun, G' \not\simeq G}} \overline{h}_{G, G'}(\cdot)$. We see that if $G'' \simeq G$, we have $\tilde{h}_G(G'') = 0$; if $G'' \not\simeq G$, then  $\tilde{h}_G(G'') > 0$.

Thus, it is left to show that $\tilde{h}_G \in \mathcal{C}^{+1}$. Following the definition of the augmented collection, we choose $d$ to be the cardinality of the set $\{ G' \in \Gfun: G \not\simeq G' \}$ and select the subset of $d$ functions, $\{h_{G, G'}\}_{G' \in \Gfun, G \not\simeq G'}$, from $\mathcal{C}$. Since
\begin{equation*}
\label{hbar}
\begin{split}
    \overline{h}_{G, G'}(\cdot) &= \max(h_{G, G'}(\cdot) - h_{G, G'}(G), 0) + \max(h_{G, G'}(G) - h_{G, G'}(\cdot), 0) 
\end{split}
\end{equation*}

we see that each $\overline{h}_{G, G'}(\cdot)$ can be obtained from $h_{G, G'}(\cdot)$ by passing through one ReLU layer. Therefore, $\tilde{h}_G \in \mathcal{C}^{+1}$.
\qed

\subsubsection{Proof of Lemma~\ref{lemma2}}
In the setting of a finite feature space, we can in fact obtain a stronger result: for all $f$ that is permutation-invariant, $f \in \mathcal{C}^{+1}$, which means no approximation is needed.

First, for every $G \in \Gfun$, we can use each $\tilde{h}_G$ to construct the indicator function $\mathds{1}_{\simeq G}$. To achieve this, as $\Gfun$ is finite, we let $\delta_G = \frac{1}{2} \min_{G' \in \Gfun, G' \not\simeq G} |\tilde{h}_G(G')| > 0$. We then introduce a ``bump'' function, $\psi_{a, b}: \mathbb{R} \to \mathbb{R}$ with parameters $a$ and $b$, defined by $\psi_{a, b}(x) \coloneqq \psi((x-b)/a)$, where $\psi(x) \coloneqq \max(x-1, 0) + \max(x+1, 0) - 2 \max(x, 0)$. We see that $\psi_{a, b}$ is nonnegative, $\psi_{a, b}(b) = 0$, and $\psi_{a, b} > 0$ only on $(b-a, b+a)$. Hence, $\psi_{\delta_G, 0}(\tilde{h}_G(\cdot))$ equals the indicator function $\mathds{1}_{\simeq G}$ on $\Gfun$.

Given any function $f$ that is permutation-invariant, as the input space $\Gfun$ is finite, we can decompose it as
$f(\cdot) = \sum_{G \in \Gfun} \frac{f(G)}{|\mathcal{E}(G)|} \mathds{1}_{\simeq G}(\cdot)$.
Thus, $f$ can be realized in $\mathcal{C}^{+1}$, as each ``$f(G)$'' on the right-hand side is treated as a constant in the neural network.
\qed

\subsection{Proof of Theorem \ref{ua2pdinf}}
\label{app:pf_thm3}
$\forall G_1, G_2 \in \Gfun$, if $G_1 \not\simeq G_2$, define a function $f_1(G) = \min_{\Pi \in S_n} \| G_1 - \Pi^\intercal \cdot G \cdot \Pi \|_{F}$. It is a continuous and permutation-invariant function on $K$, and therefore can be approximated by a function $h \in \mathcal{C}$ to within $\epsilon \coloneqq \frac{1}{2}f_1(G_2) > 0$ accuracy. Then $h$ is a function that can discriminate between $G_1$ and $G_2$. \qed

\subsection{Proof of Theorem~\ref{thm:4}}
\label{app:pf_thm4}
\begin{definition}
\label{locate}
Let $\mathcal{C}$ be a class of functions $\Gfun \to \mathbb{R}$. We say it is able to \textbf{locate every isomorphism class} if for all $G \in \Gfun$ and for all $\epsilon > 0$ there exists $h_G \in \mathcal{C}$ such that:
\begin{itemize}
    \item for all $G' \in \Gfun, h_G(G') \geq 0$;
    \item for all $G' \in \Gfun$, if $G' \simeq G$, then $h_G(G') = 0$; and
    \item there exists $\delta_G > 0$ such that for all $G' \in K$, if $h_G(G') < \delta_G$, then $\exists \Pi \in S_n$ such that $\| \Pi^\intercal \cdot G' \cdot \Pi - G \|_F < \epsilon$.
\end{itemize}
\end{definition}

Then, Theorem~\ref{thm:4} is a consequence of the two following lemmas.
\begin{lemma} \label{lemma.C+1}
Let $\mathcal{C}$ be a collection of continuous permutation-invariant functions from $\Gfun$ to $\mathbb{R}$. If $\mathcal{C}$ is GIso-discriminating, then $\mathcal{C}^{+1}$ is able to locate every isomorphism class.
\end{lemma}

\begin{lemma} \label{lemma.locate.approx}
Let $\mathcal{C}$ be a class of continuous permutation-invariant functions $\Gfun \to \mathbb{R}$. 
If $\mathcal{C}$ is able to locate every isomorphism class, then $\mathcal{C}^{+1}$ is universally approximating.
\end{lemma}

\subsubsection{Proof of Lemma \ref{lemma.C+1}}
Given $G^* \in \Gfun$ and $\epsilon > 0$, we will construct a function $h_{G^*}$ as desired in Definition~\ref{locate}. Since $\mathcal{C}$ is pairwise distinguishing, we know that $\forall G \in \Gfun$ such that $G \not\simeq G^*$, $\exists h_{G, G^*} \in \mathcal{C}$ such that $h_{G, G^*}(G) \neq h_{G, G^*}(G^*)$. For each such $G$, we define a set $A_{G, G^*}$ by
\begin{equation*}
A_{G, G^*} = \{ G' \in \Gfun : \left | h_{G, G^*}(G') - h_{G, G^*}(G) \right | < \frac{1}{2} |h_{G, G^*}(G) - h_{G, G^*}(G^*)| \}
\end{equation*}
Clearly, for each $G \not\simeq G^*$, $G \in A_{G, G^*}$ while $G^*$ does not. In addition, since $h_{G, G^*}$ is assumed to be continuous, $A_{G, G^*}$ is an open set. Meanwhile, for any $G \simeq G^*$, we define $A_{G, G^*} = B(G, \epsilon)$, the open $\epsilon$-ball centered at $G$ under the Euclidean distance. Thus, $\{ A_{G, G^*} \}_{G \in \Gfun}$ is an open cover of $\Gfun$. Since $\Gfun$ is compact, $\exists$ a finite subset $\Gfun_0$ of $\Gfun$ such that $\{ A_{G_0, G^*} \}_{G_0 \in \Gfun_0}$ also covers $\Gfun$. 
Hence, $\forall G \in \Gfun, \exists G_0 \in \Gfun_0$ such that $G \in A_{G_0, G^*}$. 

Next, we define a function $h_{G^*}$ by setting $h_{G^*}(G) = \sum_{G_0 \in \Gfun_0 \setminus \mathcal{E}(G^*)} \overline{h}_{G_0, G^*}(G), \forall G \in \Gfun$, where 
\begin{equation*}
\begin{split}
    \overline{h}_{G_0, G^*}(G) =& |h_{G_0, G^*}(G) - h_{G_0, G^*}(G^*)| \\
    =& \max(h_{G_0, G^*}(G) - h_{G_0, G^*}(G^*), 0) + \max(h_{G_0, G^*}(G^*) - h_{G_0, G^*}(G), 0)
\end{split}
\end{equation*}
Since each $h_{G_0, G^*}$ in continuous, $h_{G^*}$ is also continuous. Moreover, we can show that $h_{G^*}$ satisfies the desired properties described in Definition \ref{locate}:
\begin{itemize}
    \item Each $\bar{h}_{G_0, G^*}$ is non-negative, and hence $h_{G^*}$ is non-negative on $\Gfun$.
    \item If $G \simeq G^*$, then since each $h_{G_0, G^*}$ is permutation-invariant, there is $h_{G_0, G^*}(G^*) = h_{G_0, G^*}(G)$, and hence $\overline{h}_{G_0, G^*}(G) = 0$, $\forall G_0 \in \Gfun_0 \setminus \mathcal{E}(G^*)$. Therefore, $h_{G^*}(G) = 0$.
    \item Define $\delta_{G^*} = \frac{1}{2} \min_{G_0 \in \Gfun_0 \setminus \mathcal{E}(G^*)} |h_{G_0, G^*}(G_0) - h_{G_0, G^*}(G^*)|$.
    Suppose that $\forall \Pi \in S_n, \| \Pi^\intercal \cdot G \cdot \Pi - G^*\|_F \geq \epsilon$. This means that $G \neq \cup_{G' \in \mathcal{E}(G^*)} B(G', \epsilon) = \cup_{G' \in \mathcal{E}(G^*)} A_{G', G^*}$. Since $\{ A_{G_0, G^*}\}_{G_0 \in \Gfun_0}$ is a cover for $\Gfun$, we know that 
    $\exists G_0 \in \Gfun_0 \setminus \mathcal{E}(G^*)$ such that $G \in A_{G_0, G^*}$, and thus $|h_{G_0, G^*}(G) - h_{G_0, G^*}(G_0)| < \frac{1}{2} |h_{G_0, G^*}(G_0) - h_{G_0, G^*}(G^*)|$. Therefore, 
    \begin{equation*}
    \begin{split}
        \overline{h}_{G_0, G^*}(G) =& |h_{G_0, G^*}(G) - h_{G_0, G^*}(G^*)| \\
        \geq & |h_{G_0, G^*}(G_0) - h_{G_0, G^*}(G^*)| - |h_{G_0, G^*}(G) - h_{G_0, G^*}(G_0)| \\
        \geq & \frac{1}{2} |h_{G_0, G^*}(G_0) - h_{G_0, G^*}(G^*)| \\
        \geq & \delta_{G^*}~,
    \end{split}
    \end{equation*}
    and hence $h_{G^*}(G) \geq \delta_{G^*}$.
    Hence, in other words, if $G$ satisfies $h_{G^*}(G) < \delta_{G^*}$, then it must hold that $G \in \bigcup_{G' \in \mathcal{E}(G^*)} A_{G', G^*} = \bigcup_{\Pi \in S_n} B(\Pi^\intercal \cdot G \cdot \Pi, \epsilon)$, implying that $\exists \Pi \in S_n$ such that $\| \Pi^\intercal \cdot G \cdot \Pi - G^*\|_F < \epsilon$.
\end{itemize}
Finally, it is clear that that $h_{G^*}$ can be represented in $\mathcal{C}^{+1}$. \qed

\subsubsection{Proof of Lemma \ref{lemma.locate.approx}}
Consider any $f$ that is continuous and permutation-invariant. Since $\Gfun$ is compact, $f$ is uniformly continuous on $\Gfun$. This implies that $\forall \epsilon > 0, \exists r > 0$ such that $\forall G_1, G_2 \in \Gfun$, if $\| G_1 - G_2 \|_F < r$, then $|f(G_1) - f(G_2)| < \epsilon$. Moreover, the permutation-invariance of $f$ implies that $\forall \epsilon > 0, \exists r > 0$ such that $\forall G_1, G_2 \in \Gfun$, if $\exists \Pi \in S_n$ such that $\| \Pi^\intercal \cdot G_1 \cdot \Pi - G_2 \|_F < r$, then $|f(G_1) - f(G_2)| < \epsilon$.

Given any $G \in \Gfun$, since $\mathcal{C}$ is able to locate every isomorphism class, we can choose a function $h_G \in \mathcal{C}$ as in Definition \ref{locate}. We use $h_G^{-1}(a)$ to denote $h_G^{-1}([0, a))$. Then $\exists \delta_G$ such that $h_G^{-1}(\delta_G) \subseteq \cup_{\Pi \in S_n} B(\Pi^\intercal \cdot G \cdot \Pi, r)$, where $B(G, r)$ is the ball in $\Gfun$ centered at $G$ with radius $r$ in Euclidean distance. Since $h_G$ is continuous, $h_G^{-1}(\delta_G)$ is open. Therefore, $\{ h_G^{-1}(\delta_G) \}_{G \in \Gfun}$ is an open cover of $\Gfun$. Because $\Gfun$ is compact, $\exists$ a finite subset $\Gfun_0 \subseteq \Gfun$ such that $\{ h_{G_0}^{-1}(\delta_{G_0}) \}_{G \in \Gfun_0}$ also covers $\Gfun$.

For each $G_0 \in \Gfun_0$, we construct another function $\varphi_{G_0}$ on $\Gfun$ by defining, $\forall G \in \Gfun$,
\begin{equation*}
    \varphi_{G_0}(G) = \max\{ \delta_{G_0} - h_{G_0}(G), 0 \}~.
\end{equation*}
It is clear that $\text{supp}(\varphi_{G_0}) = h_{G_0}^{-1}(\delta_{G_0})$. Next, for each $G_0 \in \Gfun_0$, we construct another function $\psi_{G_0}$ on $\Gfun$ by defining, $\forall G \in \Gfun$,
\begin{equation*}
\label{eq:psi_G0}
    \psi_{G_0}(G) = \frac{\varphi_{G_0}(G)}{\sum_{G' \in \Gfun_0} \varphi_{G'}(G)}
\end{equation*}
Since $\{ h_{G_0}^{-1}(\delta_{G_0}) \}_{G_0 \in \Gfun_0}$ covers $\Gfun$, we know that $\forall G \in \Gfun$,  $\exists G_0 \in \Gfun_0$ such that $G \in h_{G_0}^{-1}(\delta_{G_0})$, which implies that $\varphi_{G_0}(G) > 0$, and therefore the denominator in \eqref{eq:psi_G0} is positive. Thus, for each $G_0 \in \Gfun_0$, $\psi_{G_0}$ is a well-defined function on $\Gfun$, and $\text{supp}(\psi_{G_0}) = \text{supp}(\varphi_{G_0}) = h_{G_0}^{-1}(\delta_{G_0})$. Moreover, $\forall G \in \Gfun, \sum_{G_0 \in \Gfun_0} \psi_{G_0}(G) = 1$. In other words, the set of functions $\{ \psi_{G_0} \}_{G_0 \in \Gfun_0}$ is a \emph{partition of unity} on $\Gfun$ with respect to the open cover $\{ h_{G_0}^{-1}(\delta_{G_0}) \}_{G_0 \in \Gfun_0}$.

This implies that $\forall G \in \Gfun$, there is
\begin{equation*}
    f(G) = f(G) \left (\sum_{G_0 \in \Gfun_0}  \psi_{G_0}(G) \right ) = \mathop{\sum_{G_0 \in \Gfun_0}} f(G) \psi_{G_0}(G)
\end{equation*}
Meanwhile, if $G \in h_{G_0}^{-1}(\delta_{G_0})$, then $\exists \Pi \in S_n$ such that $\| \Pi^\intercal \cdot G \cdot \Pi - G_0 \|_F < r$, which implies that $|f(G) - f(G_0)| < \epsilon$. Hence, if we construct a function $\hat{f}$ on $\Gfun$ by defining, $\forall G \in \Gfun$,
\begin{equation}
\label{eq:hatf}
    \hat{f}(G) = \sum_{G_0 \in \Gfun_0} f(G_0) \psi_{G_0}(G) = \mathop{\sum_{G_0 \in \Gfun_0}} \frac{f(G_0) \varphi_{G_0}(G)}{\sum_{G' \in \Gfun_0} \varphi_{G'}(G)}
\end{equation}
we see that 
\begin{equation}
    \sup_{G \in \Gfun} |f(G) - \hat{f}(G)| < \epsilon~,
\end{equation}
because $\forall G \in \Gfun$,
\begin{equation}
    \begin{split}
        |f(G) - \hat{f}(G) | \leq & \mathop{\sum_{G_0 \in \Gfun_0}}_{G \in h_{G_0}^{-1}(\delta_{G_0})} \left | f(G) - f(G_0) \right | \psi_{G_0}(G) \\
        =& \mathop{\sum_{G_0 \in \Gfun_0}}_{G \in h_{G_0}^{-1}(\delta_{G_0})} \epsilon \psi_{G_0}(G) \\
        <& \epsilon
    \end{split}
\end{equation}

Finally, we show how to approximate $\hat{f}$ with functions from $\mathcal{C}$ augmented with a feed-forward neural network. Given any input graph $G$, applying each of $\{h_{G_0}\}_{G_0 \in \Gfun_0} \subseteq \mathcal{C}$ to $G$ yields a vector of size $|\Gfun_0|$. Moreover, we see from \eqref{eq:hatf} that $\hat{f}(G)$ can be viewed as a continuous function of this vector. Hence, by the universal approximation theorem \citep{cybenko1989approximation, hornik1991hornik}, $\hat{f}(G)$ can be approximated to arbitrary accuracy by a feed-forward neural network applied to this vector that is sufficiently wide. \qed

\section{Proofs for Section \ref{sec.sigma}} \label{sec.proofs.reformulating}
\subsection{Proof of Theorem \ref{teo5}}
If $\mathcal{C}$ is GIso-discriminating, then given a $G \in \Gfun$, $\forall G' \not \simeq G, \exists h_{G'} \in \mathcal{C}$ and $b_{G'} \in \mathbb{R}$ such that $\mathcal{E}(G) = \cap_{G' \not \simeq G} h_{G'}^{-1}(\{ b_G' \})$, which is a finite intersection  of sets in $\sigma(\mathcal{C})$. Hence, $\mathcal{E}(G) \in \sigma(f_G) \subseteq \sigma(\mathcal{C})$. Therefore, $\Qfun \subseteq \sigma(\mathcal{C})$, and hence $\sigma(\Qfun) \subseteq \sigma(\mathcal{C})$. Moreover, since $\sigma(g) \subseteq \sigma(\Qfun)$ for all $g \in \mathcal{C}$, there is $\sigma(\mathcal{C}) \subseteq \sigma(\Qfun)$
\qed

\subsection{Proof of Theorem \ref{teo6}}
Suppose not. This implies that $\Qfun \subsetneq \sigma(\mathcal{C})$, and hence $\exists \tau = \mathcal{E}(G) \in \Qfun$ such that $\tau \notin \sigma(\mathcal{C})$. Note that $\tau$ is an equivalence class of graphs that are isomorphic to each other. Then consider the smallest subset in $\sigma(\mathcal{C})$ that contains $\tau$, defined as $\displaystyle S(\tau) = \mathop{\bigcap_{T \in \sigma(\mathcal{C})}}_{\tau \subseteq T} T.$

Since $K$ is a finite space, $\sigma(\mathcal{C})$ is also finite, and hence this is a finite intersection. Since a sigma-algebra is closed under finite intersection, there is $S(\tau) \in \sigma(C)$. As $\tau \notin \sigma(\mathcal{C})$, we know that $\tau \subsetneq S(\tau)$. Then, $\exists G' \not\simeq G$ such that $G' \in S(\tau)$. Then there does not exist any function $h$ in $\mathcal{C}$ such that $h(G) \neq h(G')$, since otherwise the pre-image of some interval in $\mathbb{R}$ under $h$ will intersect with only $\mathcal{E}(G)$ but not $\mathcal{E}(G')$. Contradiction.
\qed

\section{Comparison of the expressive power of function families on graphs via the sigma-algebras} \label{app.comparison}
\subsection{Comparing sigma-algebras}
Given two classes of functions $\mathcal{C}_1, \mathcal{C}_2$, such as two classes of GNNs, there are four possibilities regarding their relative representation power, using the language of sigma-algebra developed in the main text:

\begin{itemize}
    \item Equivalent expressive power: $\sigma(\mathcal{C}_1) = \sigma(\mathcal{C}_2)$; 
    \item $\mathcal{C}_2$ is strictly more powerful: $\sigma(\mathcal{C}_1) \subsetneq \sigma(\mathcal{C}_2)$;
    \item $\mathcal{C}_1$ is strictly more powerful: $\sigma(\mathcal{C}_2) \subsetneq \sigma(\mathcal{C}_1)$;
    \item Not comparable: $\sigma(\mathcal{C}_1) \nsubseteq \sigma(\mathcal{C}_2)$ and $\sigma(\mathcal{C}_1) \nsubseteq \sigma(\mathcal{C}_2)$.
\end{itemize}

\subsection{Details of Figure~\ref{fig.diagram}}
\label{app:diagram}
In this section, we summarize some results from the literature (including relevant results known after the original publication of this work) that allow us to establish partial relationships among the expressive power of different GNNs architectures under the sigma-algebra framework introduced in Section \ref{sec.sigma} and above. For simplicity, here we assume that graphs are determined by the adjacency matrix $A$ and no node or edge features are included. The results below are illustrated in Figure \ref{fig.diagram}.

\paragraph{Spectral GNNs (sGNNs)} Let $\Omega$ denote a set of $n \times n$ matrices that represent linear operators on graph signals. 
In a spectral GNN model with $T$ layers, we set $d_0 = 1$ and $v^{(0)} = \mathds{1}_n \in \mathbb{R}^{n \times 1}$, and recursively compute
\begin{equation*}
    v^{(t+1)} \coloneqq \rho\left(\sum_{M\in \Omega} M v^{(t)} \theta_{M}^{(t)}\right) \in \mathbb{R}^{n \times d_{t+1}}~,
\end{equation*}
where $\theta_{M}^{(t)} \in \mathbb R^{d_t\times d_{t+1}}$ is a trainable parameter matrix at each layer $t$. The model finally outputs the entry-wise sum of the vector $v^{(T)}$.
This type of models was proposed to solve community detection tasks in graphs in \cite{chen2019cdsbm}, and has also been applied to solve the quadratic assignment problem \cite{nowak2017note}. 

In this context, for any positive integer $J$, we define $\Omega_J \coloneqq \{ I, D, \hat{A}_{(1)}, ..., \hat{A}_{(J)} \}$, where $I$ is the identity matrix, $D$ is the degree matrix of the graph, $A$ is the adjacency matrix of the graph, and we define $\hat{A}_{(j)} \coloneqq \min\{A^{2^j},1\}$, which is the adjacency matrix of the $j$th power graph of $G$.
We then write $\text{sGNN}_J$ for the spectral GNN model equipped with $\Omega = \Omega_J$.

The power graph adjacency matrices leverage multi-scale information in the graph, but note that they differ from simply taking the powers of the adjacency matrix, which is exploited in \cite{liao2018lanczosnet} for example. As mentioned in Section~\ref{sec:ringgnn}, the power graph adjacency matrices allow the model to distinguish regular graphs that cannot distinguished by the $1$-WL test, such as the Circular Skip Link graphs. 

See a later work \cite{wang2022powerful} for further discussions on the expressive power of spectrally-designed GNNs.

\paragraph{k-Weisfeiler-Lehman ($k$-WL)} For positive integers $k$, $k$-WL is an iterative algorithm for determining whether a pair of graphs is isomorphic \cite{weisfeiler1968reduction}. The higher $k$ is, the more powerful $k$-WL is in distinguishing non-isomorphic graphs. In particular, $2$-WL is as powerful as $1$-WL on graphs with no edge features, but $2$-WL, unlike $1$-WL, is able to take into account edge features. When $k=1$, $1$-WL has also been called the color refinement algorithm. We note that there is an alternative ``folklore'' definition of the WL algorithm, for which we refer the readers to \cite{cai1992optimal, maron2019provably}, for example.

\paragraph{Graph Isomorphism Network (GIN) and $k$-dimensional GNN ($k$-GNN)} GIN is a GNN model proposed in \cite{xu2018powerful}, where it is proved that GIN is as powerful as $1$-WL in distinguishing non-isomorphic graphs. A similar model and its generalizations to $k$-WL with $k > 1$ are proposed in \cite{morris2019higher}, called $k$-GNN.

\paragraph{$k$-Invariant Graph Networks ($k$-IGNs)} Proposed in \cite{maron2018invariant, maron2019universality}. In later works, it is proved that $k$-IGN is exactly as powerful as $k$-WL for both $k=2$ \cite{chen2020can} and $k > 2$ \cite{geerts2020expressive, geerts2022expressiveness}.

\paragraph{Provably Powerful Graph Network (PPGN)} The PPGN model is proposed in the concurrent work of \cite{maron2019provably}. It extends the $2$-IGN model through matrix multiplication steps in a way that highly resembles Ring-GNN. In addition, it is proved in \cite{maron2019provably} that PPGN (and hence Ring-GNN) are as powerful as $3$-WL.

\paragraph{Linear Programming (LP)} The graph isomorphism problem can be formulated through an optimization problem. Namely, if $G, G' \in \mathcal{X}^{n\times n}$ are the adjacency matrices of two graphs of size $n$, we consider the optimization problem
\begin{align*}
    \min_{\Pi} & \quad \| \Pi \cdot G - G' \cdot \Pi \|_1 \\
    \text{s.t.} & \quad \Pi \in S_n~.
\end{align*}
By the definition of graph isomorphism, we see that $G$ and $G'$ are isomorphic if and only if the minimum value is zero.
Then, one may consider an \emph{LP relaxation} of the optimization problem above:
\begin{align*}
    LP(G, G') \coloneqq \min_{\Pi} & \quad \| \Pi \cdot G - G' \cdot \Pi \|_1 \\
    \text{s.t.} & \quad \Pi \succeq 0,~ P \cdot \mathds{1}_n = P^\intercal \cdot \mathds{1}_n = \mathds{1}_n~.
\end{align*}
We say the two graphs are \emph{fractionally isomorphic} if $LP(G, G') = 0$.
Thus, the LP relaxation leads to a natural sigma algebra $\sigma(\cup_{G \in \Gfun } \{LP(G, \cdot) \}).$ 
It has been shown that two graphs are fractionally isomorphic if and only if they cannot be distinguished by 1-WL \cite{tinhofer1986graph, tinhofer1991note, ramana1994fractional}. 

\paragraph{Semidefinite Programming (SDP)} 
The SDP relaxation of the quadratic assignment problem \cite{zhao1998semidefinite} is always in the feasible set of the LP, and therefore LP is less expressive than SDP.
\paragraph{Sum-of-Squares (SoS) hierarchy} One can consider the hierarchy of relaxations coming from sum-of-squares (SoS). In the context of graph isomorphism, it is known that graph isomorphism is a hard problem for this hierarchy \cite{o2014hardness}. In particular the Lasserre/SoS hierarchy requires $2^{\Omega(n)}$ to solve graph isomorphism (in the same sense that $o(n)$-WL fails to solve graph isomorphism \cite{cai1992optimal}).

\paragraph{Spectrum($A$)} If we consider the function that takes a graph and outputs the set of eigenvalues of its adjacency matrix, such a function is permutation invariant. 
On one hand, certain regular graphs can be distinguished by their adjacency spectra but not by $1$-WL or $2$-WL. On the other hands, there are non-isomorphic graphs that can be distinguished by $1$-WL / $2$-WL but share the same adjacency spectrum (e.g., Figure 2 of \cite{ramana1994fractional}). Meanwhile, $3$-WL is known to be strictly more powerful than the adjacency spectrum \cite{furer2010combinatorial, rattan2023spectra}.

\section{Additional Discussions of Literature}
\label{app.probabilistic}
The article~\cite{bloemreddy2019probabilistic} provides a nice and general theoretical framework that establishes an equivalence between the functional and probabilistic perspectives to symmetry via noise outsourcing in both general and particular settings. Our framework belongs to the functional perspective to symmetry (in particular, $\mathbb{S}_{n_2}$-invariance), and an extension to the probabilistic perspective with ideas from \cite{bloemreddy2019probabilistic} would be quite interesting. The concept of orbits also applies in our setting, and the concept of maximal invariants is related to our definition of GIso-discriminating. However, a key distinction is that being a maximal invariant is a property of \emph{functions}, whereas we define GIso-discriminating to be a property of \textit{classes} of functions. Our definition is arguably more suitable for studying the representation power of different GNN architectures, and moreover makes it possible to relate graph isomorphism tests to function approximation. Furthermore, the theory developed in Section 4 allows us to rigorously compare the expressive power of classes of GNN functions when they are \emph{not} necessarily GIso-discriminating, which is another novel contribution.


\section{Theoretical Limitation of the $2$-IGN model} \label{app.Ginvariant}
\label{order_2_G_inv}

In this section, we prove Theorem \ref{prop.Ginvariant}, which states that $2$-IGNs cannot distinguish between non-isomorphic regular graphs with the same degree.

\subsection{Defining the $2$-IGN model}
\label{app:2ign}
Here, we state our definition of $2$-IGN based on the $G$-invariant networks defined in \cite{maron2019universality}.

Suppose $A \in \mathbb{R}^{n^k \times a}$ is a tensor containing information of a graph $G$, where each entry is associated with a $k$-tuple of nodes. Then $\forall \Pi \in S_n$, we use $\pi * A$ to denote the $\mathbb{R}^{n^k \times a}$ tensor obtained from $A$ by applying the permutation represented by $\Pi$ to the node set. For example, if $k=2$ and $A\in \mathbb{R}^{n \times n}$ is a matrix containing edge features (a simple example being the adjacency matrix), then $\Pi * A = \Pi^\intercal A \Pi$.

\begin{definition}
A function $f: \mathbb{R}^{n^k \times a} \to \mathbb{R}^b$ is \textbf{invariant} if $\forall A \in \mathbb{R}^{n^k \times a}, \forall \Pi \in S_n, f(\Pi * A) = f(A)$. A function $f': \mathbb{R}^{n^k \times a} \to \mathbb{R}^{n^l \times b}$ is \textbf{equivariant} if $\forall A \in \mathbb{R}^{n^k \times a}, \forall \Pi \in S_n, f'(\Pi * A) = \Pi * f(A)$. Note that invariance is a special case of equivariance when $l = 0$.
\end{definition}

\begin{definition}
For a positive integer $k$, a $k$-IGN parameterizes a function $F: \mathbb{R}^{n^{k_{0}} \times d_{0}} \to \mathbb{R}$ in the following way:
\[
F = m \circ h \circ L^{(T)} \circ \sigma \circ \dots \circ \sigma \circ L^{(1)},
\]
where each $L^{(t)}$ is a linear equivariant layer from $\mathbb{R}^{n^{k} \times d_{t-1}}$ to $\mathbb{R}^{n^k \times d_{i}}$, $\sigma$ is a pointwise activation function, $h$ is an invariant layer from $\mathbb{R}^{n^{k} \times d_{T}}$ to $\mathbb{R}$, and $m$ is an MLP.
\end{definition}

We will use $A^{(t)}$ to denote the output of the $t$th layer, for $t \in \{1, ..., T\}$, i.e., they are defined recursively by
\[
A^{(t+1)} = \sigma(L^{(t)}(A^{(t)}))~,
\]
where $A^{(0)} \in \mathbb{R}^{n \times n \times d_0}$ is the input to the model.

\subsection{Proof of Theorem \ref{prop.Ginvariant}}

In the definition of $2$-IGN in Appendix~\ref{app:2ign}, each $d_{t}$ can be interpreted as the dimension of the hidden state at layer $t$ attached to each pair of nodes. For simplicity of notations,  we assume in the following proof that $d_{t} = 1, \forall t = 0, 1, ..., L$ (in which case each $A^{(t)}$ is essentially a matrix), but note that
the proof can be extended to the cases where $d_{t} > 1$ by adding more subscripts in the argument.

Let $G$ and $G'$ be two unweighted regular graphs with the same degree $d$, and let $A$ and $A' \in \mathbb{R}^{n \times n}$ denote their adjacency matrices (which coincide with the matrix representation of $G$ and $G'$, respectively).
To prove Theorem \ref{prop.Ginvariant}, we show that a $2$-IGN model is bound to return the same output when applied to $A$ and $A'$. For the graph $G$, we let $E \subseteq [n]^2$ denote its edge set, and further define $S \coloneqq \{(i, i): i \in [n] \} \subseteq [n]^2$ and $N \coloneqq [n]^2 \setminus (E \cup S)$. Similarly, we define $E', S'$ and $N'$ for the graph $G'$.

\begin{lemma}
\label{lem.A_ind}
$\forall t \leq T$, $\exists \xi_1^{(t)}, \xi_2^{(t)}, \xi_3^{(t)} \in \mathbb{R}$ such that $\forall i, j \in [n]$,
\begin{equation}
\label{eq:At_ind}
    A^{(t)}_{ij} = \begin{cases}
    \xi_1^{(t)}~,&~ \text{ if } (i, j) \in E~, \\
    \xi_2^{(t)}~,&~ \text{ if } (i, j) \in N~, \\
    \xi_3^{(t)}~,&~ \text{ if } (i, j) \in S~;
    \end{cases}
    \qquad 
    A'^{(t)}_{ij} = \begin{cases}
    \xi_1^{(t)}~,&~ \text{ if } (i, j) \in E'~, \\
    \xi_2^{(t)}~,&~ \text{ if } (i, j) \in N'~, \\
    \xi_3^{(t)}~,&~ \text{ if } (i, j) \in S'~.
    \end{cases}
\end{equation}
\end{lemma}
This Lemma is proved in Appendix~\ref{lem.A_ind}.

Since $h$ is an linear permutation-invariant function, we can write $h(A) = w_1 \sum_{i=1}^n A_{i,i} + w_2 \sum_{i \neq j} A_{i,j}$ for some $w_1$ and $w_2$. Thus, 
\begin{align*}
    h(A^{(T)}) =&~ w_1 \sum_{(i, j) \in E \cup N} A^{(T)}_{i, j} + w_2 \sum_{(i, j) \in S} A^{(T)}_{i, j} \\
    =&~ w_1 \big ( |E| \xi_1^{(T)} + |N| \xi_2^{(T)} \big ) + w_2 |S| \xi_3^{(T)} \\
    =&~ w_1 \big ( |E'| \xi_1^{(T)} + |N'| \xi_2^{(T)} \big ) + w_2 |S'| \xi_3^{(T)} \\
    =&~ h(A'^{(T)})~,
\end{align*}
where we use the observation that $|E| = |E'|, |S| = |S'|$ and $|N| = |N'|$. As a consequence, $F(A) = m(h(A^{(L)})) = m(h(A'^{(L)})) = F(A')$.  \qed

\subsection{Proof of Lemma~\ref{lem.A_ind}}
\label{app:pf_lem_A_ind}
If $i, i', j, j' \in [n]$, we say $(i, j)$ and $(i', j')$ are equivalent as $2$-tuples of nodes (or node pairs) if $\exists$ a permutation $\pi$ (that is, $\pi$ is a bijective map from $[n]$ to itself) such that $\pi(i) = i'$ and $\pi(j) = j'$. For any $i, j \in [n]$, we let $\psi(i, j) \subseteq [n]^2$ denote the equivalence class of $2$-tuples containing $(i, j)$. Similarly, we can define a notion of equivalence between $4$-tuples of nodes and the corresponding equivalence classes denoted by $\phi(\cdot, \cdot, \cdot, \cdot)$.

We prove this lemma by induction. For $t=0$, $A^{(0)} = A$ and $A'^{(0)} = A'$. By the definition of the adjacency matrix, $A_{ij} = 1$ if $i \neq j$ and $(i, j) \in E$, and $0$ otherwise. Similar is true for $A'$. Therefore, setting $\xi_1^{(0)} = 1$ and $\xi_2^{(0)} = \xi_3^{(0)} = 0$, it is straightforward to verify that \eqref{eq:At_ind} holds.

Next, we consider the inductive steps. Assume that \eqref{eq:hatf} is satisfied at layer $t-1$. To simplify the notation, we will let $A, A', \xi_1, \xi_2$ and $\xi_3$ stand for $A^{(t-1)}, A'^{(t-1)}, \xi_1^{(t-1)}, \xi_2^{(t-1)}, \xi_3^{(t-1)}$ below. Thus, we want to show that if $L$ is any linear permutation-equivariant layer, then $\sigma(L(A)), \sigma(L(A'))$ also satisfies the inductive hypothesis. 
Using equation 9(b) in \cite{maron2018invariant} and adopting the notations of $T, B, C, w, \beta$ defined therein, we know that $L(A)$ can be written as follows: for any $i, j \in [n]$,
\begin{equation*}
    \begin{split}
        L(A)_{i, j} &= Y_{i, j} + \sum_{i', j' = 1}^{n} T_{i', j', i, j} A_{i', j'} \\
        &= \sum_{\Psi} \beta_{\Psi} C_{i, j}^{\Psi} + \sum_{i', j', \Phi} w_{\Phi} B_{i', j', i, j}^{\Phi} A_{i', j'} \\
        &= \beta_{\psi(i, j)} + \sum_\Phi S^{\Phi}_{i, j} w_\Phi~,
    \end{split}
\end{equation*}
where $\Psi$ is summed over all equivalence classes of $2$-tuples of nodes, $\Phi$ is summed over all equivalence classes of $4$-tuples of nodes, and we define
\begin{equation*}
    S^{\Phi}_{i, j} = \mathop{\sum_{i', j' \in [n]}}_{(i', j', i, j) \in \Phi} A_{i', j'}~.
\end{equation*}
Thus, by the inductive hypothesis and the fact that $[n]^2$ equals the disjoint of $E$, $N$ and $S$, there is 
\begin{equation}
\label{eq:S}
\begin{split}
    S^{\Phi}_{i, j} &= \sum_{\substack{i', j' \in [n] \\ (i', j', i, j) \in \Phi \\ (i', j') \in E}} A_{i', j'} + \sum_{\substack{i', j' \in [n] \\ (i', j', i, j) \in \Phi \\ (i', j') \in N}} A_{i', j'} + \sum_{\substack{i', j' \in [n] \\ (i', j', i, j) \in \Phi \\ (i', j') \in S}} A_{i', j'} \\
    &= \sum_{\substack{i', j' \in [n] \\ (i', j', i, j) \in \Phi \\ (i', j') \in E}} \xi_1 + \sum_{\substack{i', j' \in [n] \\ (i', j', i, j) \in \Phi \\ (i', j') \in N}} \xi_2 + \sum_{\substack{i', j' \in [n] \\ (i', j', i, j) \in \Phi \\ (i', j') \in S}} \xi_3 \\
    &= m_E(i, j; \Phi) \xi_1 + m_N(i, j; \Phi) \xi_2 + m_S(i, j; \Phi) \xi_3
\end{split}
\end{equation}
where $m_E(i, j; \Phi)$ is defined as the total number of distinct $(i', j') \in E$ that satisfies $(i', j', i, j) \in \Phi$, and similarly for $m_N(i, j; \Phi)$ and $m_S(i, j; \Phi)$. 

Note that $m_E(i, j; \Phi)$, $m_N(i, j; \Phi)$ and $m_S(i, j; \Phi)$ depend \emph{not} on the graph structure but only on $(i, j)$ and $\Phi$. In particular, for $\alpha, \beta \in \{E, N, S\}$, there exist $m_\alpha(\beta; \cdot)$ such that we may write
\begin{equation}
\label{eq:m}
    m_\alpha(i, j; \Phi) = \begin{cases}
        m_\alpha(E; \Phi)~,& \text{ if } (i, j) \in E~, \\
        m_\alpha(N; \Phi)~,& \text{ if } (i, j) \in N~, \\
        m_\alpha(S; \Phi)~,& \text{ if } (i, j) \in S~,
    \end{cases}
\end{equation}
for each $\alpha \in \{E, N, S\}$. The functions $m_\alpha(\beta; \cdot)$ depend on $n$ and $d$ and can be computed in a combinatorial fashion, whose values are completely enumerated in 
Tables~\ref{table:mE}, \ref{table:mN} and \ref{table:mS}. An illustration of their meanings is given in Figure \ref{coloredreg}.
\begin{figure}
\label{coloredreg}
    \centering
    \includegraphics[width=0.25\textwidth,trim={6cm 4cm 20cm 7cm},clip]{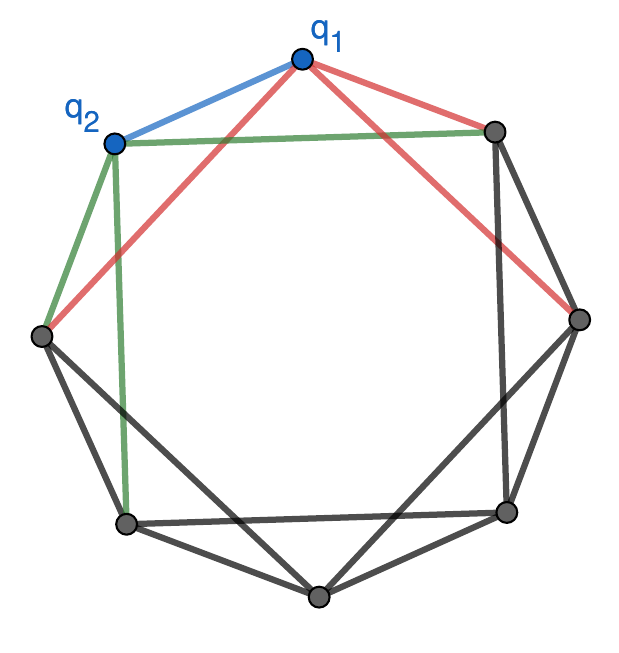}
    \includegraphics[width=0.25\textwidth,trim={6cm 4cm 20cm 7cm},clip]{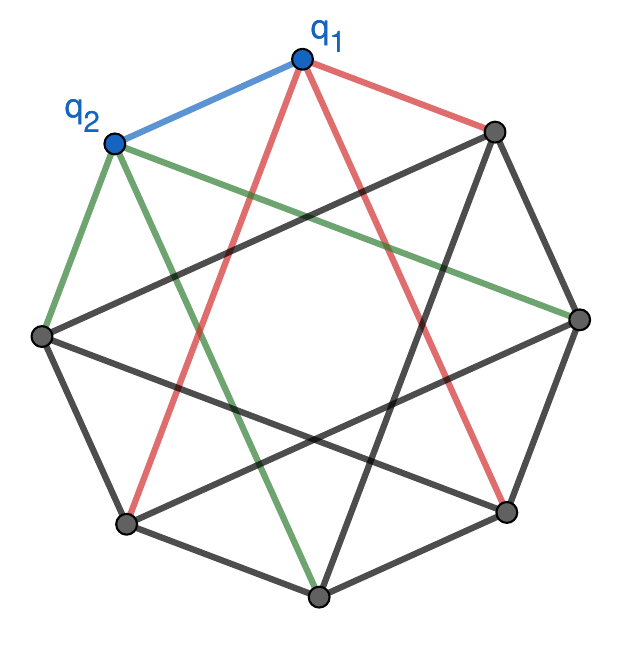}
    \caption{
    Illustrations of 
    $m_E(E, \cdot)$
    on graphs $G_{8, 2} $ and $G_{8, 3}$. In either graph, $m_E(E, \Phi(1, 2, 3, 4)) = 18$ represents the total number of node pairs $(i, j)$ such that $(i, j) \in E$ and $(i, j, q_1, q_2)$ is equivalent to $(1, 2, 3, 4)$ as $4$-tuples, and equals twice the total number of black edges (twice because if $(i, j)$ satisfies the condition, then so does $(j, i)$, even though they represent the same undirected edge); similarly, the total number of of red edges, $3$, equals both $m_E(E, \Phi(1, 2, 2, 3))$ and $m_E(E, \Phi(1, 2, 1, 3))$; the total number of green edges, also $3$, equals both $m_E(E, \Phi(1, 2, 3, 2))$, $m_E(E, \Phi(1, 2, 3, 1))$.}
    \label{fig:my_label}
\end{figure}

Combining \eqref{eq:S} and \eqref{eq:m}, we derive that
\begin{equation*}
    S^{\Phi}_{i, j} = \begin{cases}
        m_E(E; \Phi) \xi_1 + m_N(E; \Phi) \xi_2 + m_S(E; \Phi) \xi_3~,& \text{ if } (i, j) \in E~, \\
        m_E(N; \Phi) \xi_1 + m_N(N; \Phi) \xi_2 + m_S(N; \Phi) \xi_3~,& \text{ if } (i, j) \in N~, \\
        m_E(S; \Phi) \xi_1 + m_N(S; \Phi) \xi_2 + m_S(S; \Phi) \xi_3~,& \text{ if } (i, j) \in S~.
    \end{cases}
\end{equation*}
Moreover, note that $\beta_{\psi(i, j)}$ depends only on the equivalence class to which $(i, j)$ belongs, and hence we may write
\begin{equation*}
    \beta_{\psi(i, j)} = \begin{cases}
        \beta_1~,& \text{ if } (i, j) \in E \text{ or } N~, \\
        \beta_2~,& \text{ if } (i, j) \in S~,
    \end{cases}
\end{equation*}
for some $\beta_1$ and $\beta_2$. Hence, we can now write
\begin{align*}
    L(A)_{i, j} = \begin{cases}
        \bar{\xi}_1 \coloneqq \beta_1 + \sum_\Phi \big ( m_E(E; \Phi) \xi_1 + m_N(E; \Phi) \xi_2 + m_S(E; \Phi) \xi_3 \big ) w_\Phi~,& \text{ if } (i, j) \in E~, \\
        \bar{\xi}_2 \coloneqq \beta_1 + \sum_\Phi \big ( m_E(N; \Phi) \xi_1 + m_N(N; \Phi) \xi_2 + m_S(N; \Phi) \xi_3 \big ) w_\Phi~,& \text{ if } (i, j) \in N~, \\
        \bar{\xi}_3 \coloneqq \beta_2 + \sum_\Phi \big ( m_E(S; \Phi) \xi_1 + m_N(S; \Phi) \xi_2 + m_S(S; \Phi) \xi_3 \big ) w_\Phi~,& \text{ if } (i, j) \in S~.
    \end{cases}
\end{align*}
With a similar argument, we can derive that 
\begin{align*}
    L(A')_{i, j} = \begin{cases}
        \bar{\xi}_1~,& \text{ if } (i, j) \in E'~, \\
        \bar{\xi}_2~,& \text{ if } (i, j) \in N'~, \\
        \bar{\xi}_3~,& \text{ if } (i, j) \in S'~.
    \end{cases}
\end{align*}
Hence, we have shown that \eqref{eq:At_ind} holds at layer $t$ with $\xi_1^{(t)} = \sigma(\bar{\xi}_1)$, $\xi_2^{(t)} = \sigma(\bar{\xi}_2)$ and $\xi_3^{(t)} = \sigma(\bar{\xi}_3)$. \qed

\begin{table}[h]
\centering
\begin{tabular}{llll}
\multicolumn{1}{l|}{$\mu$}        & $m_E(E, \mu)$ & $m_E(N, \mu)$ & $m_E(S, \mu)$ \\ \hline
\multicolumn{1}{l|}{(1, 2, 3, 4)} & $(n-4)d+2$      & $(n-4)d$        & 0             \\
\multicolumn{1}{l|}{(1, 1, 2, 3)} & 0             & 0             & 0             \\
\multicolumn{1}{l|}{(1, 2, 2, 3)} & $d-1$           & $d$             & 0             \\
\multicolumn{1}{l|}{(1, 2, 1, 3)} & $d-1$           & $d$             & 0             \\
\multicolumn{1}{l|}{(1, 2, 3, 2)} & $d-1$           & $d$             & 0             \\
\multicolumn{1}{l|}{(1, 2, 3, 1)} & $d-1$           & $d$             & 0             \\
\multicolumn{1}{l|}{(1, 1, 1, 2)} & 0             & 0             & 0             \\
\multicolumn{1}{l|}{(1, 1, 2, 1)} & 0             & 0             & 0             \\
\multicolumn{1}{l|}{(1, 2, 1, 2)} & 1             & 0             & 0             \\
\multicolumn{1}{l|}{(1, 2, 2, 1)} & 1             & 0             & 0             \\
\multicolumn{1}{l|}{(1, 2, 3, 3)} & 0             & 0             & $(n-2)d$        \\
\multicolumn{1}{l|}{(1, 1, 2, 2)} & 0             & 0             & 0             \\
\multicolumn{1}{l|}{(1, 2, 2, 2)} & 0             & 0             & $d$             \\
\multicolumn{1}{l|}{(1, 2, 1, 1)} & 0             & 0             & $d$             \\
\multicolumn{1}{l|}{(1, 1, 1, 1)} & 0             & 0             & 0             \\ \hline
\multicolumn{1}{l|}{Total}                             & $nd$            & $nd$            & $nd$           
\end{tabular}
\caption{$m_E$}
\label{table:mE}
\end{table}

\begin{table}[h]
\centering
\begin{tabular}{llll}
\multicolumn{1}{l|}{$\mu$}        & $m_N(E, \mu)$ & $m_N(N, \mu)$ & $m_N(S, \mu)$ \\ \hline
\multicolumn{1}{l|}{(1, 2, 3, 4)} & $(n-4)(n-d-1)$      & $(n-4)(n-d-1) + 2$        & 0             \\
\multicolumn{1}{l|}{(1, 1, 2, 3)} & 0             & 0             & 0             \\
\multicolumn{1}{l|}{(1, 2, 2, 3)} & $n-d-1$           &$ n-d-2$             & 0             \\
\multicolumn{1}{l|}{(1, 2, 1, 3)} & $n-d-1$           & $n-d-2$             & 0             \\
\multicolumn{1}{l|}{(1, 2, 3, 2)} & $n-d-1$           & $n-d-2$             & 0             \\
\multicolumn{1}{l|}{(1, 2, 3, 1)} & $n-d-1$           & $n-d-2$             & 0             \\
\multicolumn{1}{l|}{(1, 1, 1, 2)} & 0             & 0             & 0             \\
\multicolumn{1}{l|}{(1, 1, 2, 1)} & 0             & 0             & 0             \\
\multicolumn{1}{l|}{(1, 2, 1, 2)} & 0             & 1             & 0             \\
\multicolumn{1}{l|}{(1, 2, 2, 1)} & 0             & 1             & 0             \\
\multicolumn{1}{l|}{(1, 2, 3, 3)} & 0             & 0             & $(n-2)(n-d-1)$        \\
\multicolumn{1}{l|}{(1, 1, 2, 2)} & 0             & 0             & 0             \\
\multicolumn{1}{l|}{(1, 2, 2, 2)} & 0             & 0             & $n-d-1$             \\
\multicolumn{1}{l|}{(1, 2, 1, 1)} & 0             & 0             & $n-d-1$             \\
\multicolumn{1}{l|}{(1, 1, 1, 1)} & 0             & 0             & 0             \\ \hline
\multicolumn{1}{l|}{Total}                             & $n(n-d-1)$            & $n(n-d-1)$            & $n(n-d-1)$           
\end{tabular}
\caption{$m_N$}
\label{table:mN}
\end{table}

\begin{table}[h]
\centering
\begin{tabular}{llll}
\multicolumn{1}{l|}{$\mu$}        & $m_S(E, \mu)$ & $m_S(N, \mu)$ & $m_S(S, \mu)$ \\ \hline
\multicolumn{1}{l|}{(1, 2, 3, 4)} & 0           & 0        & 0             \\
\multicolumn{1}{l|}{(1, 1, 2, 3)} & $n-2$             & $n-2$             & 0             \\
\multicolumn{1}{l|}{(1, 2, 2, 3)} & 0           & 0             & 0             \\
\multicolumn{1}{l|}{(1, 2, 1, 3)} & 0           & 0             & 0             \\
\multicolumn{1}{l|}{(1, 2, 3, 2)} & 0           & 0             & 0             \\
\multicolumn{1}{l|}{(1, 2, 3, 1)} & 0           & 0             & 0             \\
\multicolumn{1}{l|}{(1, 1, 1, 2)} & 1             & 1             & 0             \\
\multicolumn{1}{l|}{(1, 1, 2, 1)} & 1             & 1             & 0             \\
\multicolumn{1}{l|}{(1, 2, 1, 2)} & 0             & 0             & 0             \\
\multicolumn{1}{l|}{(1, 2, 2, 1)} & 0             & 0             & 0             \\
\multicolumn{1}{l|}{(1, 2, 3, 3)} & 0             & 0             & 0        \\
\multicolumn{1}{l|}{(1, 1, 2, 2)} & 0             & 0             & $n-1$             \\
\multicolumn{1}{l|}{(1, 2, 2, 2)} & 0             & 0             & 0             \\
\multicolumn{1}{l|}{(1, 2, 1, 1)} & 0             & 0             & 0             \\
\multicolumn{1}{l|}{(1, 1, 1, 1)} & 0             & 0             & 1             \\ \hline
\multicolumn{1}{l|}{Total}                             & $n$            & $n$            & $n $          
\end{tabular}
\caption{$m_S$}
\label{table:mS}
\end{table}

\newpage
\section{GNN Models in the Experiments}
\label{archi}
In section \ref{experiments}, we show experiments on synthetic and real datasets with several different GNN architectures. Here are some additional details of these models.
\begin{itemize}
\item $\textbf{sGNN}_J$: The spectral GNN models defined in Appendix~\ref{app:diagram}, where we choose $J = 1, 2$ and $5$.
The models have 5 layers and hidden layer dimension (i.e. $d_t$) 64. They are trained using the Adam~\cite{kingma2014adam} optimizer with learning rate 0.01.
\item \textbf{LGNN}: Line Graph Neural Networks proposed by \cite{chen2019cdsbm}. The model has 5 layers and hidden layer dimension 64. They are trained using the Adam optimizer with learning rate 0.01.
\item \textbf{GIN}: Graph Isomorphism Network by \cite{xu2018powerful}. We take their performance on the IMDB datasets reported in \cite{xu2018powerful}, and their performance on classifying the CSL graphs reported in \cite{murphy2019relational} .
\item \textbf{RP-GIN}: Graph Isomorphism Network combined with Relational Pooling by \cite{murphy2019relational}. We took the results reported in \cite{murphy2019relational} for the CSL graphs experiment.
\item \textbf{$2$-IGN}: The model defined in Appendix based on \cite{maron2018invariant} and \cite{maron2019universality} and implemented in \url{https://github.com/Haggaim/InvariantGraphNetworks}.
\item \textbf{Ring-GNN}: The definition is given in the main text. For the experiments on the IMDB datasets, the \emph{Ring-GNN} model has the same depth and widths of hidden layers as the $2$-IGN model adopted in \cite{maron2018invariant}. The \emph{Ring-GNN w/ degree} model has 2 layers with 64 hidden units in each, followed by a jump knowledge network~\cite{xu2018representation}, which is then followed by a fully-connected layer with 32 hidden units. Each $k_1^{(t)}$ is initialized independently under $\mathcal{N}(0,1)$, and each $k_2^{(t)}$ is initialized independently under $\mathcal{N}(0,0.01)$. They are trained using the Adam~\cite{kingma2014adam} optimizer with learning rate 0.00001 for 350 epochs. The initialization of $k_2^{(t)}$ and the learning rate were manually tuned, following the observation that Ring-GNN reduces to $2$-IGN when $k_2^{(t)}=0$.
For the other real-world datasets, models are trained via Adam with learning rate of 0.001 for 350 epochs. The model has 1 layer for MUTAG, 2 layers for PROTEINS and PTC, and 3 layers for COLLAB. Each of these layers has 64 hidden units and is followed by a jump knowledge network~\cite{xu2018representation}, which is then followed by a fully-connected layer with 32 hidden units. $k_1^{(t)}$ is initialized to be 1, and $k_2^{(t)}$ is initialized with $\{0.5/n, 1.0/n\}$, where $n$ is the average number of nodes per graph in each dataset.
\end{itemize}

For the experiments with CSL graphs, each model is trained and evaluated using 5-fold cross-validation. For Ring-GNN, we perform training plus cross-validation 20 times with different random seeds.

\end{document}